\documentclass[11pt]{article} 
\usepackage{fullpage}
\usepackage{hyperref}
\usepackage{url}

\usepackage{amsfonts}
\usepackage{amsmath}
\usepackage{graphicx}
\usepackage{color}

\usepackage{algorithm}
\usepackage{algorithmic}
\usepackage{ifthen}
\usepackage[algo2e,ruled]{algorithm2e}

\newtheorem{theorem}{Theorem}
\newtheorem{lemma}[theorem]{Lemma}
\newtheorem{fact}[theorem]{Fact}
\newtheorem{corollary}[theorem]{Corollary}

\newcommand{\qed}{\hfill $\Box$}
\newenvironment{proof}{\par\noindent{\bf Proof.}}{\qed \par\smallskip\noindent}

\newcommand{\field}[1]{\mathbb{#1}}

\newcommand{\E}{\field{E}}
\renewcommand{\Pr}{\field{P}}
\newcommand{\Ind}[1]{\field{I}{\{#1\}}}

\newcommand{\dt}{\displaystyle}

\newcommand{\loss}{\ell}
\newcommand{\hloss}{\widehat{\loss}}

\newcommand{\gain}{g}
\newcommand{\hgain}{\widehat{\gain}}

\newcommand{\hp}{\widehat{p}}
\newcommand{\hs}{\widehat{s}}
\newcommand{\hd}{\widehat{d}}
\newcommand{\hP}{\widehat{P}}
\newcommand{\hG}{\widehat{G}}
\newcommand{\hS}{\widehat{S}}

\newcommand{\reach}[1]{\xrightarrow{{#1}}}
\newcommand{\gammab}{\gamma^{(b)}}
\newcommand{\gammabt}{\gamma^{(b_t)}}
\newcommand{\Tb}{T^{(b)}}

\newcommand{\mas}{\mbox{\tt mas}}

\newcommand{\subsecref}[1]{Subsection~\ref{#1}}

\renewcommand{\eqref}[1]{Eq.~(\ref{#1})}
\newcommand{\lemref}[1]{Lemma~\ref{#1}}

\newcommand{\thmref}[1]{Thm.~\ref{#1}}

\newcommand{\algref}[1]{Algorithm~\ref{#1}}

\newcommand{\Ocal}{\mathcal{O}}

\title{
Nonstochastic Multi-Armed Bandits\\ with Graph-Structured Feedback
\footnote{Preliminary versions of this manuscript appeared in \cite{AlCGM13,MS11}.}
}

\date{}

\author{
Noga Alon\\
Tel Aviv University, Tel Aviv, Israel,\\
Institute for Advanced Study, Princeton, United States,\\
and
Microsoft Research, Herzliya, Israel\\
\texttt{nogaa@tau.ac.il}
\and
Nicol\`o Cesa-Bianchi\\
Universit\`a degli Studi di Milano, Italy \\
\texttt{nicolo.cesa-bianchi@unimi.it}
\and
Claudio Gentile\\
University of Insubria, Italy\\
\texttt{claudio.gentile@uninsubria.it}
\and
Shie Mannor\\
Technion, 
Israel\\
\texttt{shie@ee.technion.ac.il}
\and
Yishay Mansour \\
Microsoft Research and\\
Tel-Aviv University, Israel\\
\texttt{mansour@tau.ac.il}
\and
Ohad Shamir \\
Weizmann Institute of Science, Israel\\
\texttt{ohad.shamir@weizmann.ac.il}
}

\begin{document}

\maketitle

\begin{abstract}
We present and study a partial-information model of online learning, where
a decision maker repeatedly chooses from a finite set of actions, and
observes some subset of the associated losses. This naturally models
several situations where the losses of different actions are related, and
knowing the loss of one action provides information on the loss of other
actions. Moreover, it generalizes and interpolates between the well studied
full-information setting (where all losses are revealed) and the bandit
setting (where only the loss of the action chosen by the player is
revealed). We provide several algorithms addressing different variants of
our setting, and provide tight regret bounds depending on
combinatorial properties of the information feedback structure.
\end{abstract}

\section{Introduction}
%

Prediction with expert advice ---see, e.g.,
\cite{cb+97,cbl06,FreundSc95,LittlestoneWa94,vo90}--- is a general abstract
framework for studying sequential decision problems. For example, consider a
weather forecasting problem, where each day we receive predictions from
various experts, and we need to devise our forecast. At the end of the day,
we observe how well each expert did, and we can use this information to
improve our forecasting in the future. Our goal is that over time, our
performance converges to that of the best expert in hindsight. More formally,
such problems are often modeled as a repeated game between a player and an
adversary, where each round, the adversary privately assigns a loss value to
each action in a fixed set (in the example above, the discrepancy in the
forecast if we follows a given expert's advice). Then the player chooses an
action (possibly using randomization), and incurs the corresponding loss. The
goal of the player is to control regret, which is defined as the cumulative
excess loss incurred by the player as compared to the best fixed action over
a sequence of rounds.

In some situations, however, the player only gets partial feedback on the
loss associated with each action. For example, consider a web advertising
problem, where every day one can choose an ad to display to a user, out of a
fixed set of ads. As in the forecasting problem, we sequentially choose
actions from a given set, and may wish to control our regret with respect to
the best fixed ad in hindsight. However, while we can observe whether a
displayed ad was clicked on, we do not know what would have happened if we
chose a different ad to display. In our abstract framework, this corresponds
to the player observing the loss of the action picked, but not the losses of
other actions. This well-known setting is referred to as the (non-stochastic)
multi-armed bandit problem, which in this paper we denote as the
\emph{bandit} setting. In contrast, we refer to the previous setting, where
the player observes the losses of all actions, as the \emph{expert} setting.
In this work, our main goal is to bridge between these two feedback settings,
and create a spectrum of models in between.

Before continuing, let us first quantify the performance attainable in the
expert and the bandit setting. Letting $K$ be the number of available
actions, and $T$ be the number of played rounds, the best possible regret for
the expert setting is of order $\sqrt{\ln (K)\,T}$. This optimal rate is
achieved by the Hedge algorithm~\cite{FreundSc95} or the Follow the Perturbed
Leader algorithm~\cite{Kalai:05}. In the bandit setting, the optimal regret
is of order $\sqrt{KT}$, achieved by the INF
algorithm~\cite{DBLP:conf/colt/AudibertB09}. A bandit variant of Hedge,
called Exp3~\cite{AuerCeFrSc02}, achieves a regret with a slightly worse
bound of order $\sqrt{K\ln(K)\,T}$. Thus, switching from the full-information
expert setting to the partial-information bandit setting increases the
attainable regret by a multiplicative factor of $\sqrt{K}$, up to extra
logarithmic factors.
This exponential difference in terms of the dependence on $K$ can be crucial
in problems with large action sets. The intuition for this difference in
performance has long been that in the bandit setting, we only get $1/K$ of
the information obtained in the expert setting (as we observe just a single
loss, rather than all $K$ at each round), hence the additional $K$-factor
under the square root in the bound.

While the bandit setting received much interest, it can be criticized for not
capturing additional side-information we often have on the losses of the
different actions. As a motivating example, consider the problem of web
advertising mentioned earlier. In the standard multi-armed bandits setting,
we assume that we have no information whatsoever on whether undisplayed ads
would have been clicked on.
However, in many relevant cases, the semantic relationship among actions
(ads) implies that we do indeed have some side-information. For instance, if
two ads $i$ and $j$ are for similar vacation packages in Hawaii, and ad $i$
was displayed and clicked on by some user, it is likely that the other ad $j$
would have been clicked on as well. In contrast, if ad $i$ is for high-end
running shoes, and ad $j$ is for wheelchair accessories, then a user who
clicked on one ad is unlikely to click on the other. This sort of
side-information is not captured by the standard bandit setting. A similar
type of side-information arises in product recommendation systems hosted on
online social networks, in which users can befriend each other. In this case,
it has been observed that social relationships reveal similarities in tastes
and interests~\cite{said2010social}. Hence, a product liked by some user may
also be liked by the user's friends. A further example, not in the marketing
domain, is route selection: We are given a graph of possible routes
connecting cities. When we select a route connecting two cities, we observe
the cost (say, driving time or fuel consumption) of the ``edges'' along that
route and, in addition, we have complete information on sub-routes including
any subset of the edges.\footnote{ Though this example may also be viewed as
an instance of combinatorial bandits~\cite{Cesa-BianchiL12}, the model we
propose is more general. For example, it does not assume linear losses, which
could arise in the routing example from the partial ordering of sub-routes. }

In this paper, we present and study a setting which captures these types of
side-information, and in fact interpolates between the bandit setting and the
expert setting. This is done by defining a \emph{feedback system}, under
which choosing a given action also reveals the losses of some subset of the
other actions. This feedback system can be viewed as a directed and
time-changing graph $G_t$ over actions: an arc (directed edge) from action
$i$ to action $j$ implies that when playing action $i$ at round $t$ we get
information also about the loss of action $j$ at round $t$. Thus, the expert
setting is obtained by choosing a complete graph over actions (playing any
action reveals all losses), and the bandit setting is obtained by choosing an
empty edge set (playing an action only reveals the loss of that action). The
attainable regret turns out to depend on non-trivial combinatorial properties
of this graph. To describe our results, we need to make some distinctions in
the setting that we consider.

\paragraph{Directed vs.\ symmetric setting.} In some situations, the
      side-information between two actions is symmetric ---for example, if
      we know that both actions will have a similar loss. In that case, we
      can model our feedback system $G_t$ as an undirected graph. In
      contrast, there are situations where the side-information is not
      symmetric. For example, consider the side-information gained from
      asymmetric social links, such as followers of celebrities. In such
      cases, followers might be more likely to shape their preferences
      after the person they follow, than the other way around. Hence, a
      product liked by a celebrity is probably also liked by his/her
      followers, whereas a preference expressed by a follower is more often
      specific to that person. Another example in the context of ad
      placement is when a person buying a video game console might also buy a
      high-def cable to connect it to the TV set. Vice versa, interest in
      high-def cables need not indicate an interest in game consoles. In
      such situations, modeling the feedback system via a directed graph
      $G_t$ is more suitable. Note that the symmetric setting is a special
      case of the directed setting, and therefore handling the symmetric
      case is easier than the directed case.
\paragraph{Informed vs.\ uninformed setting.} In some cases, the
      feedback system is known to the player before each round, and can
      be utilized for choosing actions. For example, we may know beforehand
      which pairs of ads are related, or we may know the users who are friends of another user.
      We denote this setting as the informed setting. In
      contrast, there might be cases where the player does not have full
      knowledge of the feedback system before choosing an action, and we denote this harder setting as the uninformed setting. For example, consider a firm recommending products to users of an online social network. If the network is owned by a third party, and therefore not fully visible, the system may still be able to run its recommendation policy by only accessing small portions of the social graph around each chosen action (i.e., around each user to whom a recommendation is sent).

\medskip
Generally speaking, our contribution lies in both characherizing the regret
bounds that can be achieved in the above settings as a function of
combinatorial properties of the feedback systems, as well as providing
efficient sequential decision algorithms working in those settings. More
specifically, our contributions can be summarized as follows (see
Section~\ref{s:prel} for a brief review of the relevant combinatorial
properties of graphs).
\paragraph{Uninformed setting.}
We present an algorithm (Exp3-SET) that achieves $
    \widetilde{\Ocal}\Bigl(\sqrt{\ln(K)\sum_{t=1}^T \mas(G_t)}\Bigr)
$
      regret in expectation, where
      $\mas(G_t)$ is the
      size of the maximal acyclic graph in $G_t$. In the symmetric setting,
      $\mas(G_t)=\alpha(G_t)$ ($\alpha(G_t)$ is the independence number of $G_t$), and we
      prove that the resulting regret bound is optimal up to logarithmic
      factors, when $G_t=G$ is fixed for all rounds. Moreover, we show that
      Exp3-SET attains $\Ocal\bigl(\sqrt{\ln(K)\,T}\bigr)$ regret when the
      feedback graphs $G_t$ are random graphs generated from a standard
      Erd\H{o}s-Renyi model.
\paragraph{Informed setting.} We present an algorithm (Exp3-DOM) that
achieves expected regret of $
    \Ocal\Bigl(\ln(K)\sqrt{\ln(KT)\sum_{t=1}^T\alpha(G_t)}\Bigr)
$, for both the symmetric
      and directed cases. Since our lower bound also applies to the
      informed setting, this characterizes the attainable regret in the
      informed setting, up to logarithmic factors. Moreover, we present
      another algorithm (ELP.P), that achieves
$ \Ocal\Bigl(\sqrt{\ln(K/\delta)\sum_{t=1}^T \mas(G_t)}\Bigr) $
      regret
      with probability at least $1-\delta$ over the algorithm's internal
      randomness. Such a high-probability guarantee is stronger than the
      guarantee for Exp3-DOM, which holds just in expectation, and turns out to be
      of the same order in the symmetric case. However, in the directed case,
      the regret bound may be weaker since $\mas(G_t)$ may be larger than
      $\alpha(G_t)$. Moreover, ELP.P requires us to solve a linear program
      at each round, whereas Exp3-DOM only requires finding an approximately
      minimal dominating set, which can be done by a standard greedy set
      cover algorithm.

\medskip
Our results interpolate between the bandit and expert settings: When $G_t$ is
a full graph for all $t$ (which means that the player always gets to see all
losses, as in the expert setting), then $\mas(G_t)=\alpha(G_t)=1$, and we
recover the standard guarantees for the expert setting: $\sqrt{T}$ up to
logarithmic factors. In contrast, when $G_t$ is the empty graph for all $t$
(which means that the player only observes the loss of the action played, as
in the bandit setting), then $\mas(G_t)=\alpha(G_t)=K$, and we recover the
standard $\sqrt{KT}$ guarantees for the bandit setting, up to logarithmic
factors. In between are regret bounds scaling like $\sqrt{BT}$, where $B$
lies between $1$ and $K$, depending on the graph structure (again, up to
log-factors).

Our results are based on the algorithmic framework for handling the standard
bandit setting introduced in \cite{AuerCeFrSc02}. In this framework, the
full-information Hedge algorithm is combined with unbiased estimates of the
full loss vectors in each round. The key challenge is designing an
appropriate randomized scheme for choosing actions, which correctly balances
exploration and exploitation or, more specifically, ensures small regret
while simultaneously controlling the variance of the loss estimates. In our
setting, this variance is subtly intertwined with the structure of the
feedback system. For example, a key quantity emerging in the analysis of
Exp3-DOM can be upper bounded in terms of the independence number of the
graphs. This bound (Lemma~\ref{l:weightedamlemma} in the appendix) is based
on a combinatorial construction which may be of independent interest.

For the uninformed setting, our work was recently improved by \cite{KNVM14},
whose main contribution is an algorithm attaining
$\Ocal\Bigl(\sqrt{\ln(K)\ln(KT)\sum_{t=1}^{T}\alpha(G_t)}\Bigr)$ expected
regret in the uninformed and directed setting using a novel implicit
exploration idea. Up to log factors, this matches the performance of our
Exp3-DOM and ELP.P algorithms, without requiring prior knowledge of the
feedback system. On the other hand, their bound holds only in expectation
rather than with high probability.

\paragraph{Paper Organization:}
In the next section, we formally define our learning protocols, introduce our
main notation, and recall the combinatorial properties of graphs that we
require. In Section~\ref{s:symm}, we tackle the uninformed setting, by
introducing Exp3-SET, with upper and lower bounds on regret based on both the
size of the maximal acyclic subgraph (general directed case) and the
independence number (symmetric case). In Section~\ref{s:directed}, we handle
the informed setting through the two algorithms Exp3-DOM
(Section~\ref{ss:expdom}) on which we prove regret bounds in expectation, and
ELP.P (Section~\ref{sec:elpp}) whose bounds hold in the more demanding high
probability regime. We conclude the main text with Section~\ref{s:conc},
where we discuss open questions, and possible directions for future research.
All technical proofs are provided in the appendices. We organized such proofs
based on which section of the main text the corresponding theoretical claims
occur.

\section{Learning protocol, notation, and preliminaries}\label{s:prel}
As stated in the introduction, we consider adversarial decision problems with
a finite action set $V = \{1,\dots,K\}$. At each time $t=1,2,\dots$, a player
(the ``learning algorithm'') picks some action $I_t \in V$ and incurs a
bounded loss $\loss_{I_t,t} \in [0,1]$. Unlike the adversarial bandit
problem~\cite{AuerCeFrSc02,cbl06}, where only the played action $I_t$ reveals
its loss $\loss_{I_t,t}$, here we assume all the losses in a subset
$S_{I_t,t} \subseteq V$ of actions are revealed after $I_t$ is played. More
formally, the player observes the pairs $(i,\loss_{i,t})$ for each $i \in
S_{I_t,t}$. We also assume $i\in S_{i,t}$ for any $i$ and $t$, that is, any
action reveals its own loss when played. Note that the bandit setting
($S_{i,t} = \{i\}$) and the expert
setting ($S_{i,t} = V$) 
are both special cases of this framework. We call $S_{i,t}$ the {\em feedback
set} of action $i$ at time $t$, and write $i \reach{t} j$ when at time $t$
playing action $i$ also reveals the loss of action $j$. (We sometimes write
$i \reach{} j$ when time $t$ plays no role in the surrounding context.) With
this notation, $S_{i,t} = \{j\in V\,:\, i \reach{t} j\}$. The family of
feedback sets $\{S_{i,t}\}_{i\in V}$ we collectively call the {\em feedback
system} at time $t$.

The adversaries we consider are nonoblivious. Namely, each loss $\loss_{i,t}$
and feedback set $S_{i,t}$ at time $t$ can be arbitrary functions of the past
player's actions $I_1,\dots,I_{t-1}$ (note, though, that the regret is
measured with respect to a fixed action assuming the adversary would have
chosen the same losses, so our results do not extend to truly adaptive
adversaries in the sense of \cite{DBLP:conf/icml/DekelTA12}). The performance
of a player $A$ is measured through the expected regret
\[
    \max_{k\in V} \E\bigl[L_{A,T} - L_{k,T}\bigl]
\]
where $L_{A,T} = \loss_{I_1,1} + \cdots + \loss_{I_T,T}$ and $L_{k,T} =
\loss_{k,1} + \cdots + \loss_{k,T}$
are the cumulative losses of the player and of action $k$,
respectively.\footnote { Although we defined the problem in terms of losses,
our analysis can be applied to the case when actions return rewards $g_{i,t}
\in [0,1]$ via the transformation $\loss_{i,t} = 1 - g_{i,t}$. } The
expectation is taken with respect to the player's internal randomization
(since losses are allowed to depend on the player's past random actions,
$L_{k,T}$ may also be random). In Section~\ref{s:symm} we also consider a
variant in which the feedback system is randomly generated according to a
specific stochastic model.
For simplicity, we focus on a finite horizon setting, where the number of
rounds $T$ is known in advance. This can be easily relaxed using a standard
doubling trick.

We also consider the harder setting where the goal is to bound the actual
regret
\[
     L_{A,T} - \max_{k\in V} L_{k,T}
\]
with high probability $1-\delta$ with respect to the player's internal
randomization, and where the regret bound depends logarithmically on
$1/\delta$. Clearly, a high probability bound on the actual regret implies a
similar bound on the expected regret.

Whereas some of our algorithms need to know the feedback system at the
beginning of each step $t$, others need it only at the end of each step. We
thus consider two online learning settings: the {\em informed} setting, where
the full feedback system $\{S_{i,t}\}_{i\in V}$ selected by the adversary is
made available to the learner {\em before} making the choice $I_t$;
and the {\em uninformed setting}, where no information whatsoever regarding
the time-$t$ feedback system is given to the learner prior to prediction, but
only following the prediction and with the associated loss information.

We find it convenient at this point to adopt a graph-theoretic interpretation
of feedback systems.
At each step $t=1,2,\dots,T$, the feedback system $\{S_{i,t}\}_{i\in V}$
defines a directed graph $G_t = (V,D_t)$, the feedback graph, where $V$
is the set of actions and $D_t$ is the set of arcs (i.e., ordered pairs of
nodes). For $j \neq i$, the arc $(i,j)$ belongs to $D_t$ if and only if $i
\reach{t} j$ (the self-loops created by $i \reach{t} i$ are intentionally
ignored). Hence, we can equivalently define $\{S_{i,t}\}_{i\in V}$ in terms
of $G_t$. Observe that the outdegree $d_{i,t}^+$ of any $i \in V$ equals
$|S_{i,t}|-1$. Similarly, the indegree $d_{i,t}^-$ of $i$ is the number of
actions $j \neq i$ such that $i \in S_{j,t}$ (i.e., such that $j
\reach{t}i$).
A notable special case of the above is when the feedback system is symmetric:
$j \in S_{i,t}$ if and only if $i \in S_{j,t}$ for all $i,j$ and $t$. In
words, playing $i$ at time $t$ reveals the loss of $j$ if and only if playing
$j$ at time $t$ reveals the loss of $i$. A symmetric feedback system defines
an undirected graph $G_t$ or, more precisely, a directed graph having, for
every pair of nodes $i,j \in V$, either no arcs or length-two directed
cycles.
Thus, from the point of view of the symmetry of the feedback system, we also
distinguish between the {\em directed} case ($G_t$ is a general directed
graph) and the {\em symmetric} case ($G_t$ is an undirected graph for all
$t$).

The analysis of our algorithms depends on certain properties of the sequence
of graphs $G_t$. Two graph-theoretic notions playing an important role here
are those of {\em independent sets} and {\em dominating sets}. Given an
undirected graph $G = (V,E)$, an independent set of $G$ is any subset $T
\subseteq V$ such that no two $i,j \in T$ are connected by an edge in $E$,
i.e., $(i,j)\not\in E$. An independent set is {\em maximal} if no proper
superset thereof is itself an independent set. The size of any largest (and
thus maximal) independent set is the {\em independence number} of $G$,
denoted by $\alpha(G)$. If $G$ is directed, we can still associate with it an
independence number: we simply view $G$ as undirected by ignoring arc
orientation. If $G = (V,D)$ is a directed graph, then a subset $R \subseteq
V$ is a dominating set for $G$ if for all $j \not\in R$ there exists some $i
\in R$ such that $(i,j) \in D$. In our bandit setting, a time-$t$ dominating
set $R_t$ is a subset of actions with the property that the loss of any
remaining action in round $t$ can be observed by playing some action in
$R_t$. A dominating set is {\em minimal} if no proper subset thereof is
itself a dominating set. The domination number of directed graph $G$, denoted
by $\gamma(G)$, is the size of any smallest (and therefore minimal)
dominating set for $G$; see Figure~\ref{f:1} for examples.
\begin{figure}[t!]
\begin{picture}(-40,220)(-40,220)
\scalebox{0.6}{\includegraphics{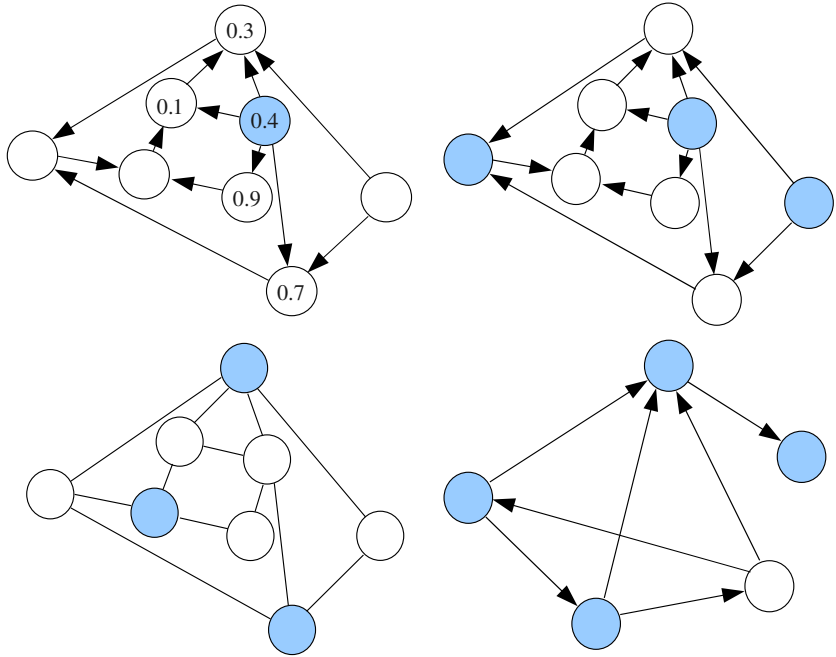}}
\end{picture}
\vspace{0.6in}
\caption{An example for some graph-theoretic concepts. {\bf Top Left:} A feedback system with $K = 8$ actions (self-loops omitted).
The light blue action reveals its loss $0.4$, as well as the losses of the other four actions it points to.
{\bf Top Right:} The light blue nodes are a minimal dominating set for the same graph.
The rightmost action is included in any dominating set, since no other action is dominating it.
{\bf Bottom Left:} A symmetric feedback system where the light blue nodes are a maximal independent set. This is the same graph as before, but edge orientation has been removed. {\bf Bottom Right:} The light blue nodes are a maximum acyclic subgraph of the depicted $5$-action graph.
\label{f:1}
}
\end{figure}

Computing a minimum dominating set for an arbitrary directed graph $G_t$ is
equivalent to solving a minimum set cover problem on the associated feedback
system $\{S_{i,t}\}_{i\in V}$. Although minimum set cover is NP-hard, the
well-known Greedy Set Cover algorithm~\cite{Chv79}, which repeatedly selects
from $\{S_{i,t}\}_{i\in V}$ the set containing the largest number of
uncovered elements so far, computes a dominating set $R_t$ such that $|R_t|
\leq \gamma(G_t)\,(1+\ln K)$.

We can also lift the notion of independence number of an undirected graph to
directed graphs through the notion of {\em maximum acyclic subgraphs}. Given
a directed graph $G = (V,D)$, an acyclic subgraph of $G$ is any graph $G' =
(V',D')$ such that $V'\subseteq V$, and $D'= D\cap \bigl(V'\times V'\bigr)$,
with no (directed) cycles. We denote by $\mas(G)= |V'|$ the maximum size of
such $V'$. Note that when $G$ is undirected (more precisely, as above, when
$G$ is a directed graph having for every pair of nodes $i,j \in V$ either no
arcs or length-two cycles), then $\mas(G) = \alpha(G)$, otherwise $\mas(G)
\geq \alpha(G)$. In particular, when $G$ is itself a directed acyclic graph,
then $\mas(G) = |V|$. See Figure~\ref{f:1} (bottom right) for a simple
example. Finally, we let $\Ind{A}$ denote the indicator function of event
$A$.

\section{The uninformed setting}\label{s:symm}
%

In this section we investigate the setting in which the learner must select
an action without any knowledge of the current feedback system. We introduce
a simple general algorithm, Exp3-SET (Algorithm~\ref{a:lossalg}), that works
in both the directed and symmetric cases. In the symmetric case, we show that
the regret bound achieved by the algorithm is optimal to within logarithmic
factors.

When the feedback graph $G_t$ is a fixed clique or a fixed edgeless graph,
Exp3-SET reduces to the Hedge algorithm or, respectively, to the Exp3
algorithm. Correspondingly, the regret bound for Exp3-SET yields the regret
bound of Hedge and that of Exp3 as special cases.

\begin{algorithm2e}[t]
\SetKwSty{textrm} \SetKwFor{For}{{\bf For}}{}{}
\SetKwIF{If}{ElseIf}{Else}{if}{}{else if}{else}{} \SetKwFor{While}{while}{}{}
\textbf{Parameter:} $\eta \in [0,1]$\\
\textbf{Initialize:} $w_{i,1} = 1$ for all $i \in V = \{1,\ldots,K\}$\\
\For{$t=1,2,\dots$:} { {
\begin{enumerate}
\item Feedback system $\{S_{i,t}\}_{i\in V}$ and losses $\loss_t$ are
    generated but not disclosed~;
\item Set ${\dt p_{i,t} = \frac{w_{i,t}}{W_t}}$ for each $i\in V$, where
    ${\dt W_t = \sum_{j \in V} w_{j,t}}$~;
\item Play action $I_t$ drawn according to distribution $p_t =
    (p_{1,t},\dots,p_{K,t})$~;
\item Observe:
\begin{enumerate}
\item pairs $(i,\loss_{i,t})$ for all $i \in S_{I_t,t}$;
\item Feedback system $\{S_{i,t}\}_{i\in V}$ is disclosed;
\end{enumerate}
\item For any $i \in V$ set $w_{i,t+1} =
    w_{i,t}\,\exp\bigl(-\eta\,\hloss_{i,t}\bigr)$, where
\[
    \hloss_{i,t}
=
    \frac{\loss_{i,t}}{q_{i,t}}\,\Ind{i \in S_{I_t,t}}
\qquad\text{and}\qquad
    q_{i,t} = \sum_{j \,:\, j \reach{t} i} p_{j,t}~.
\]
\end{enumerate}
} } \caption{The Exp3-SET algorithm (for the uninformed setting)}
\label{a:lossalg}
\end{algorithm2e}
%
Similar to Exp3, Exp3-SET uses importance sampling loss estimates
$\hloss_{i,t}$ that divide each observed loss $\loss_{i,t}$ by the
probability $q_{i,t}$ of observing it. This probability $q_{i,t}$ is the
probability of observing the loss of action $i$ at time $t$, i.e., it is
simply the sum of all $p_{j,t}$ (the probability of selecting action $j$ at
time $t$) such that $j \reach{t} i$ (recall that this sum always includes
$p_{i,t}$).

In the expert setting, we have $q_{i,t} = 1$ for all $i$ and $t$, and we
recover the Hedge algorithm. In the bandit setting, $q_{i,t} = p_{i,t}$ for
all $i$ and $t$, and we recover the Exp3 algorithm (more precisely, we
recover the variant Exp3Light of Exp3 that does not have an explicit
exploration term, see \cite{Cesa-BianchiMS05} and also
\cite[Theorem~2.7]{Stoltz2005}).

In what follows, we show that the regret of Exp3-SET can be bounded in terms
of the key quantity
\begin{equation}\label{e:Qt}
    Q_t = \sum_{i \in V} \frac{p_{i,t}}{q_{i,t}} = \sum_{i \in V} \frac{p_{i,t}}{\sum_{j \,:\, j \reach{t} i} p_{j,t}}~.
\end{equation}
Each term $p_{i,t}/q_{i,t}$ can be viewed as the probability of drawing $i$
from $p_t$ conditioned on the event that $\loss_{i,t}$ was observed.
A key aspect to our analysis is the ability to deterministically and
non-vacuously\footnote{An obvious upper bound on $Q_t$ is $K$, since
$p_{i,t}/q_{i,t}\leq 1$.} upper bound $Q_t$ in terms of certain quantities
defined on $\{S_{i,t}\}_{i\in V}$. We do so in two ways, either irrespective
of how small each $p_{i,t}$ may be (this section) or depending on suitable
lower bounds on the probabilities $p_{i,t}$ (Section~\ref{s:directed}). In
fact, forcing lower bounds on $p_{i,t}$ is equivalent to adding exploration
terms to the algorithm, which can be done only when $\{S_{i,t}\}_{i\in V}$ is
known before each prediction (i.e., in the informed setting).


The following result, whose proof is in Appendix~\ref{s:appendix:thm:noexp},
is the building block for all subsequent results in the uninformed setting.
\begin{lemma}\label{thm:noexp}
The regret of Exp3-SET satisfies
\begin{equation}
\label{eq:set-bound}
    \max_{k \in V} \E\bigl[L_{A,T} - L_{k,T}\bigr]
\le
    \frac{\ln K}{\eta} + \frac{\eta}{2}\,\sum_{t=1}^T \E[Q_t]~.
\end{equation}
\end{lemma}
In the expert setting, $q_{i,t}=1$ for all $i$ and $t$ implies $Q_t = 1$
deterministically for all $t$. Hence, the right-hand side
of~(\ref{eq:set-bound}) becomes $
    ({\ln K})/{\eta} + ({\eta}/{2})\,T~,
$ corresponding to the Hedge bound with a slightly larger constant in the
second term; see, e.g., \cite[Page~72]{cbl06}. In the bandit setting,
$q_{i,t}=p_{i,t}$ for all $i$ and $t$ implies $Q_t = K$ deterministically for
all $t$. Hence, the right-hand side of~(\ref{eq:set-bound}) takes the form $
({\ln K})/{\eta} + ({\eta}/{2})\,KT~, $ equivalent to the Exp3 bound; see,
e.g., \cite[Equation~3.4]{bubeck2012regret}.

We now move on to the case of general feedback systems, for which we can
prove the following result (proof is in Appendix~\ref{s:appendix:c:ndag}).
\begin{theorem}\label{c:ndag}
The regret of Exp3-SET satisfies
\[
    \max_{k \in V} \E\bigl[L_{A,T} - L_{k,T}\bigr]
\le
    \frac{\ln K}{\eta} + \frac{\eta}{2}\,\sum_{t=1}^T \E[\mas(G_t)]~.
\]
If $\mas(G_t) \leq m_t$ for $t=1,\dots,T$, then setting $\eta = \sqrt{(2\ln
K)\big/\sum_{t=1}^T m_t}$ gives
\[
    \max_{k \in V} \E\bigl[L_{A,T} - L_{k,T}\bigr]
\le
    \sqrt{2(\ln K)\sum_{t=1}^T m_t}~.
\]
\end{theorem}
As we pointed out in Section~\ref{s:prel}, $\mas(G_t) \geq \alpha(G_t)$, with
equality holding when $G_t$ is an undirected graph. Hence, in the special
case when $G_t$ is symmetric, we obtain the following result.
%
\begin{corollary}\label{thm:symmetric}
In the symmetric case, the regret of Exp3-SET satisfies
\[
    \max_{k \in V} \E\bigl[L_{A,T} - L_{k,T}\bigr]
\le
    \frac{\ln K}{\eta} + \frac{\eta}{2}\,\sum_{t=1}^T \E[\alpha(G_t)]~.
\]
If $\alpha(G_t) \leq \alpha_t$ for $t = 1, \ldots, T$, then setting $\eta =
\sqrt{(2\ln K)\big/\sum_{t=1}^T \alpha_t}$ gives
\[
    \max_{k \in V} \E\bigl[L_{A,T} - L_{k,T}\bigr]
\le
    \sqrt{2(\ln K)\sum_{t=1}^T \alpha_t}~.
\]
\end{corollary}
Note that both Theorem \ref{c:ndag} and Corollary \ref{thm:symmetric} require
the algorithm to know upper bounds on $\mas(G_t)$ and $\alpha(G_t)$, which
may be computationally non-trivial --  we return and expand on this issue in
section \ref{sec:elpp}.

In light of~Corollary~\ref{thm:symmetric}, one may wonder whether
Lemma~\ref{thm:noexp} is powerful enough to allow a control of regret in
terms of the independence number even in the directed case. Unfortunately,
the next result shows that ---in the directed case--- $Q_t$ cannot be
controlled unless specific properties of $p_t$ are assumed. More precisely,
we show that even for simple directed graphs, there exist distributions $p_t$
on the vertices such that $Q_t$ is linear in the number of nodes while the
independence number\footnote { In this specific example, the maximum acyclic
subgraph has size $K$, which confirms the looseness of Theorem~\ref{c:ndag}.
} is $1$.
\begin{fact}
\label{l:bad} Let $G = (V,D)$ be a total order on $V = \{1,\dots,K\}$, i.e.,
such that for all $i \in V$, arc $(j,i) \in D$ for all $j = i+1,\dots,K$. Let
$p = (p_1, \ldots, p_K)$ be a distribution on $V$ such that $p_i=2^{-i}$, for
$i<K$ and $p_k=2^{-K+1}$. Then
\[
Q = \sum_{i = 1}^K \frac{p_i}{p_i+\sum_{j\,:\,j\reach{}i} p_j} =
\sum_{i = 1}^K \frac{p_i}{\sum_{j=i}^K p_j} = \frac{K+1}{2}~.
\]
\end{fact}
Next, we discuss lower bounds on the achievable regret for arbitrary
algorithms. The following theorem provides a lower bound on the regret in
terms of the independence number $\alpha(G)$, for a constant graph $G_t=G$
(which may be directed or undirected).
\begin{theorem}\label{th:lowerbound}
Suppose $G_t = G$ for all $t$ with $\alpha(G)>1$. There exist two constants
$C_1,C_2 > 0$ such that whenever $T \ge C_1\alpha(G)^3$, then for any
algorithm there exists an adversarial strategy for which the expected regret
of the algorithm is at least $C_2\sqrt{\alpha(G) T}$.
\end{theorem}
%
The intuition of the proof (provided in
Appendix~\ref{subsec:prooflowerbound}) is the following: if the graph $G$ has
$\alpha(G)$ non-adjacent vertices, then an adversary can make this problem as
hard as a standard bandit problem, played on $\alpha(G)$ actions. Since for
bandits on $K$ actions there is a $\Omega(\sqrt{KT})$ lower bound on the
expected regret, a variant of the proof technique leads to a
$\Omega(\sqrt{\alpha(G) T})$ lower bound in our case.

One may wonder whether a sharper lower bound exists which applies to the
general directed adversarial setting and involves the larger quantity
$\mas(G)$. Unfortunately, the above measure does not seem to be related to
the optimal regret: using Lemma~\ref{cl:1} in Appendix~\ref{sa:erdos} (see
proof of Theorem~\ref{thm:random_er} below) one can exhibit a sequence of
graphs each having a large acyclic subgraph, on which the regret of Exp3-SET
is still small.

\paragraph{Random feedback systems.}
We close this section with a study of Lemma~\ref{thm:noexp} in a setting
where the feedback system is stochastically generated via the Erd\H{o}s-Renyi
model. This is a standard model for random directed graphs $G = (V,D)$, where
we are given a density parameter $r \in [0,1]$ and, for any pair $i,j \in V$,
arc $(i,j)\in D$ with independent probability $r$ (self loops, i.e., arcs
$(i,i)$ are included by default here). We have the following result.
\begin{theorem}\label{thm:random_er}
For $t=1,2,\dots$, let $G_t$ be an independent draw from the Erd\H{o}s-Renyi
model with fixed parameter $r \in [0,1]$. Then the regret of Exp3-SET
satisfies
\[
    \max_{k \in V} \E\bigl[L_{A,T} - L_{k,T}\bigr]
\le
    \frac{\ln K}{\eta} + \frac{\eta\,T}{2r}\Bigl(1 - (1-r)^{K}\Bigr)~.
\]
In the above, expectations are computed with respect to both the algorithm's
randomization and the random generation of $G_t$ occurring at each round. In
particular, setting $\eta = \sqrt{\frac{2r\ln K}{T\bigl(1 -
(1-r)^{K}\bigr)}}$, gives
\[
    \max_{k \in V} \E\bigl[L_{A,T} - L_{k,T}\bigr]
\le
    \sqrt{\frac{2(\ln K)T\bigl(1 - (1-r)^{K}\bigr)}{r}}~.
\]
\end{theorem}
Note that as $r$ ranges in $[0,1]$ we interpolate between the multi-arm
bandit\footnote{ Observe that $\lim_{r\rightarrow 0^+} \frac{1 -
(1-r)^{K}}{r} = K$. } ($r=0$) and the expert ($r=1$) regret bounds.

Finally, note that standard results from the theory of Erd\H{o}s-Renyi graphs
---at least in the symmetric case (see, e.g., \cite{f90})--- show that when the
density parameter $r$ is constant, the independence number of the resulting
graph has an inverse dependence on $r$. This fact, combined with the lower
bound above, gives a lower bound of the form $\sqrt{{T}/{r}}$, matching (up
to logarithmic factors) the upper bound of Theorem~\ref{thm:random_er}.


\section{The informed setting}\label{s:directed}
The lack of a lower bound matching the upper bound provided by
Theorem~\ref{c:ndag} is a good indication that something more sophisticated
has to be done in order to upper bound the key quantity $Q_t$ defined
in~(\ref{e:Qt}). This leads us to consider more refined ways of allocating
probabilities $p_{i,t}$ to nodes. We do so by taking advantage of the
informed setting, in which the learner can access $G_t$ before selecting the
action $I_t$. The algorithm Exp3-DOM, introduced in this section, exploits
the knowledge of $G_t$ in order to achieve an optimal (up to logarithmic
factors) regret bound.

Recall the problem uncovered by Fact~\ref{l:bad}: when the graph induced by
the feedback system is directed, $Q_t$ cannot be upper bounded, in a
non-vacuous way, independent of the choice of probabilities $p_{i,t}$. The
new algorithm Exp3-DOM controls these probabilities by adding an exploration
term to the distribution $p_t$. This exploration term is supported on a
dominating set of the current graph $G_t$, and computing such a dominating
set before selection of the action at time $t$ can only be done in the
informed setting. Intuitively, exploration on a dominating set allows to
control $Q_t$ by increasing the probability $q_{i,t}$ that each action $i$ is
observed. If the dominating set is also minimal, then the variance caused by
exploration can be bounded in terms of the independence number (and
additional logarithmic factors) just like the undirected case.

Yet another reason why we may need to know the feedback system beforehand is
when proving high probability results on the regret. In this case, operating
with a feedback term for the probabilities $p_{i,t}$ seems unavoidable. In
Section~\ref{sec:elpp} we present another algorithm, called ELP.P,
which can deliver regret bounds that hold with high probability over its
internal randomization.

\begin{algorithm2e}[t]
\SetKwSty{textrm} \SetKwFor{For}{{\bf For}}{}{}
\SetKwIF{If}{ElseIf}{Else}{if}{}{else if}{else}{} \SetKwFor{While}{while}{}{}
{\bf Input:} Exploration parameters $\gammab \in (0,1]$ for $b \in \bigl\{0,1,\ldots, \lfloor \log_2 K\rfloor\bigr\}$\\
{\bf Initialization:} $w^{(b)}_{i,1} = 1$ for all $i \in V = \{1, \ldots,
K\}$
and $b \in \bigl\{0,1,\ldots, \lfloor \log_2 K\rfloor\bigr\}$\\
\For{$t=1,2,\dots$ :} { {
\begin{enumerate}
\item Feedback system $\{S_{i,t}\}_{i\in V}$ is generated {\em and
    disclosed}, (losses $\loss_t$ are generated and not disclosed);
\item Compute a dominating set $R_t\subseteq V$ for $G_t$ associated with
    $\{S_{i,t}\}_{i\in V}$~;
\item Let $b_t$ be such that $|R_t| \in \bigl[2^{b_t},2^{b_t+1}-1\bigr]$;
\item Set $W^{(b_t)}_t = \sum_{i \in V} w^{(b_t)}_{i,t};$
\item Set ${\dt p^{(b_t)}_{i,t} = \bigl(1-\gammabt\bigr)
    \frac{w^{(b_t)}_{i,t}}{W^{(b_t)}_{t}} + \frac{\gammabt}{|R_t|}
    \Ind{i\in R_t} }$;
\item Play action $I_t$ drawn according to distribution $p^{(b_t)}_t =
    \bigl(p^{(b_t)}_{1,t},\dots,p^{(b_t)}_{K,t}\bigr)$~;
\item Observe pairs $(i,\loss_{i,t})$ for all $i \in S_{I_t,t}$;
\item For any $i \in V$ set $w^{(b_t)}_{i,t+1} =
    w^{(b_t)}_{i,t}\,\exp\bigl(-\gammabt\,\hloss^{(b_t)}_{i,t}/2^{b_t}\bigr)$,
    where
\[
    \hloss^{(b_t)}_{i,t}
=
    \frac{\loss_{i,t}}{q^{(b_t)}_{i,t}}\,\Ind{i \in S_{I_t,t}}
\qquad\text{and}\qquad
    q^{(b_t)}_{i,t} = \sum_{j \,:\, j \reach{t} i} p^{(b_t)}_{j,t}~.
\]
\end{enumerate}
} } \caption{The Exp3-DOM algorithm (for the informed setting)}
\label{a:exp3dom}
\end{algorithm2e}

\subsection{Bounds in expectation: the Exp3-DOM algorithm}\label{ss:expdom}
The Exp3-DOM algorithm (see Algorithm~\ref{a:exp3dom}) for the informed
setting runs $\mathcal{O}(\log K)$ variants of Exp3 (with explicit
exploration) indexed by $b = 0,1,\dots,\lfloor \log_2 K\rfloor$. At time $t$
the algorithm is given the current feedback system $\{S_{i,t}\}_{i\in V}$,
and computes a dominating set $R_t$ of the directed graph $G_t$ induced by
$\{S_{i,t}\}_{i\in V}$. Based on the size $|R_t|$ of $R_t$, the algorithm
uses instance $b_t = \lfloor \log_2|R_t|\rfloor$ to draw action $I_t$. We use
a superscript $b$ to denote the quantities relevant to the variant of Exp3
indexed by $b$. Similarly to the analysis of Exp3-SET, the key quantities are
\[
    q^{(b)}_{i,t} = \sum_{j \,:\, i \in S_{j,t}} p^{(b)}_{j,t} = \sum_{j \,:\, j \reach{t} i} p^{(b)}_{j,t}
\qquad\text{and}\qquad
    Q^{(b)}_t = \sum_{i \in V} \frac{p^{(b)}_{i,t}}{q^{(b)}_{i,t}}~,\qquad b = 0, 1, \ldots, \lfloor \log_2 K\rfloor~.
\]
Let $\Tb = \bigl\{ t=1,\dots,T \,:\, |R_t| \in [2^b,2^{b+1}-1] \bigr\}$.
Clearly, the sets $\Tb$ are a partition of the time steps $\{1,\dots,T\}$, so
that $\sum_b |\Tb| = T$. Since the adversary adaptively chooses the
dominating sets $R_t$ (through the adaptive choice of the feedback system at
time $t$), the sets $\Tb$ are random variables.
This causes a problem in tuning the parameters $\gammab$. For this reason, we
do not prove a regret bound directly for Exp3-DOM, where each instance uses a
fixed $\gammab$, but for a slight variant of it (described in the proof of
Lemma~\ref{thm:alg} --- see Appendix~\ref{sa:lemma_exp3_dom}), where each
$\gammab$ is set through a doubling trick.
%
\begin{lemma}\label{thm:alg}
In the directed case, the regret of Exp3-DOM satisfies
\begin{equation}
\label{eq:gammabfixed}
    \max_{k\in V} \E\bigl[L_{A,T} - L_{k,T}\bigr]
\le
    \sum_{b=0}^{\lfloor\log_2 K\rfloor} \left( \frac{2^b\ln K}{\gammab}
    +  \gammab\E\left[\sum_{t \in \Tb} \left(1 + \frac{Q^{(b)}_t}{2^{b+1}}\right)\right]\right)~.
\end{equation}
Moreover, if we use a doubling trick to choose $\gammab$ for each $b =
0,\dots,\lfloor \log_2 K\rfloor$, then
\begin{equation}
\label{eq:doublingtrick}
  \max_{k\in V} \E\bigl[L_{A,T} - L_{k,T}\bigr]
=
{\mathcal O}\left((\ln K)\,\E\left[\sqrt{\sum_{t=1}^T \left(4|R_t|
+ Q_t^{(b_t)}\right)}\right] + (\ln K) \ln(KT)\right)~.
\end{equation}
\end{lemma}
Importantly, the next result (proof in Appendix~\ref{s:appendix:c:final})
shows how bound~(\ref{eq:doublingtrick}) of Lemma~\ref{thm:alg} can be
expressed in terms of the sequence $\alpha(G_t)$ of independence numbers of
graphs $G_t$ whenever the Greedy Set Cover algorithm~\cite{Chv79} (see
Section~\ref{s:prel}) is used to compute the dominating set $R_t$ of the
feedback system at time $t$.
\begin{theorem}\label{c:final}
If Step~2 of Exp3-DOM uses the Greedy Set Cover algorithm to compute the
dominating sets $R_t$, then the regret of Exp-DOM using the doubling trick
satisfies
\[
   \max_{k\in V} \E\bigl[L_{A,T} - L_{k,T}\bigr]
=
    \mathcal{O}\left(\ln(K)\sqrt{\ln(KT)\sum_{t=1}^T \alpha(G_t)} + \ln(K)\ln(KT) \right)~.
\]
\end{theorem}
Combining the upper bound of Theorem~\ref{c:final} with the lower bound of
Theorem~\ref{th:lowerbound}, we see that the attainable expected regret in
the informed setting is characterized by the independence numbers of the
graphs. Moreover, a quick comparison between Corollary~\ref{thm:symmetric}
and Theorem~\ref{c:final} reveals that a symmetric feedback system overcomes
the advantage of working in an informed setting: The bound we obtained for
the uninformed symmetric setting (Corollary~\ref{thm:symmetric}) is sharper
by logarithmic factors than the one we derived for the informed
 --- but more general, i.e., directed --- setting (Theorem~\ref{c:final}).

\subsection{High probability bounds: the ELP.P algorithm}\label{sec:elpp}
We now turn to present an algorithm working in the informed setting for which
we can also prove high-probability regret bounds.\footnote { We have been
unable to prove high-probability bounds for Exp3-DOM or variants of it. } We
call this algorithm ELP.P (which stands for ``Exponentially-weighted
algorithm with Linear Programming'', with high Probability). Like Exp3-DOM,
the exploration component is not uniform over the actions, but is chosen
carefully to reflect the graph structure at each round. In fact, the optimal
choice of the exploration for ELP.P requires us to solve a simple linear
program, hence the name of the algorithm.\footnote{ We note that this
algorithm improves over the basic ELP algorithm initially presented in
\cite{MS11}, in that its regret is bounded in high probability and not just
in expectation, and applies in the directed case as well as the symmetric
case.} The pseudo-code appears as \algref{alg:bandits}. Note that unlike the
previous algorithms, this algorithm utilizes the ``rewards'' formulation of
the problem, i.e., instead of using the losses $\ell_{i,t}$ directly, it uses
the rewards $g_{i,t}=1-\ell_{i,t}$, and boosts the weight of actions for
which $g_{i,t}$ is estimated to be large, as opposed to decreasing the weight
of actions for which $\ell_{i,t}$ is estimated to be large. This is done
merely for technical convenience, and does not affect the complexity of the
algorithm nor the regret guarantee.

%
\begin{algorithm2e}[t]
\SetKwSty{textrm} \SetKwFor{For}{{\bf For}}{}{}
\SetKwIF{If}{ElseIf}{Else}{if}{}{else if}{else}{} \SetKwFor{While}{while}{}{}
\textbf{Input:} Confidence parameter $\delta \in (0,1)$, learning rate $\eta > 0$;\\
{\bf Initialization:} $w_{i,1} = 1$ for all $i \in V = \{1,\ldots, K\}$;\\
\For{$t=1,2,\dots$ :} { {
\begin{enumerate}
\item Feedback system $\{S_{i,t}\}_{i\in V}$ is generated {\em and
    disclosed}, (losses $\loss_t$ are generated and not disclosed);
\item Let $\Delta_K$ be the $K$-dimensional probability simplex, and $s_t =
    (s_{1,t}, \ldots s_{K,t})$
be a solution to\\
the linear program
        \[
        \max_{ (s_1,\ldots,s_K) \in \Delta_K }~~\min_{i\in V}\sum_{j\,:\, j \reach{t} i } s_ {j}
        \]
\item Set $p_{i,t}:=(1-\gamma_t)\frac{w_{i,t}}{W_t}+\gamma_t s_{i,t}$ where
    $W_t = \sum_{i \in V} w_{i,t}$~,
\[
\gamma_t = \frac{(1+\beta)\,\eta}{\min_{i\in V}\sum_{j\,:\, j \reach{t} i} s_{j,t}} \qquad\text{and}\qquad \beta=2\eta\sqrt{\frac{\ln(5K/\delta)}{\ln K}}~;
\]
\item Play action $I_t$ drawn according to distribution $p_t =
    \bigl(p_{1,t},\dots,p_{K,t}\bigr)$~;
\item Observe pairs $(i,\loss_{i,t})$ for all $i \in S_{I_t,t}$;
\item For any $i \in V$ set $\gain_{i,t}=1-\loss_{i,t}$ and $w_{i,t+1} =
    w_{i,t}\,\exp\bigl(\eta\,\hgain_{i,t}\bigr)$, where
\[
    \hgain_{i,t}
=
    \frac{\gain_{i,t}\Ind{i \in S_{I_t,t}} + \beta}{q_{i,t}}
     \qquad\text{and}\qquad q_{i,t} = \sum_{j \,:\, j \reach{t} i} p_{j,t}~.
\]
\end{enumerate}
} } \caption{The ELP.P algorithm (for the informed setting)}
\label{alg:bandits}
\end{algorithm2e}
%
%
%
\begin{theorem}\label{thm:mainhighprob}
Let algorithm ELP.P
run with learning rate $\eta\leq 1/(3K)$ sufficiently small such that $\beta
\leq 1/4$.
%
%
Then, with probability at least $1-\delta$ we have
\begin{align*}
L_{A,T} - \max_{k\in V} L_{k,T}
~\leq~
&\sqrt{5\ln\left(\frac{5}{\delta}\right)\sum_{t=1}^{T}\mas(G_t)}
+\frac{2\ln (5K/\delta)}{\eta} + 12\eta\,\sqrt{\frac{\ln(5K/\delta)}{\ln K}}\,\sum_{t=1}^{T}\mas(G_t)\\
&~+\widetilde{\Ocal}\left(1+\sqrt{T\eta}+T\eta^2\right)\left(\max_{t = 1...T}\,\mas^2(G_t)\right),
\end{align*}
where the $\widetilde{\Ocal}$ notation hides only numerical constants and
factors logarithmic in $K$ and $1/\delta$.
In particular, if for constants $m_1, \ldots, m_T$ we have $\mas(G_t) \leq
m_t$, $t = 1, \ldots, T$, and we pick $\eta$ such that
\[
\eta^2 = \frac{1}{6}\,\frac{\sqrt{\ln(5K/\delta)\,(\ln K)}}{\sum_{t=1}^{T} m_t}
\]
then we get the bound
\[
L_{A,T} - \max_{k\in V} L_{k,T}
~\leq~
10\,\frac{\ln^{1/4}(5K/\delta)}{\ln^{1/4}K}\,\sqrt{\ln \left(\frac{5K}{\delta}\right)\sum_{t=1}^{T} m_t}
+\widetilde{\Ocal}(T^{1/4})\left(\max_{t = 1...T} \mas^2(G_t) \right)~.
\]
\end{theorem}
%
%
This theorem essentially tells us that the regret of the ELP.P algorithm, up
to second-order factors, is quantified by $\sqrt{\sum_{t=1}^{T}\mas(G_t)}$.
Recall that, in the special case when $G_t$ is symmetric, we have $\mas(G_t)
=\alpha(G_t)$.

One computational issue to bear in mind is that this theorem (as well as
Theorem \ref{c:ndag} and Corollary \ref{thm:symmetric}) holds under an
optimal choice of $\eta$.
In turn, this value depends on upper bounds on $\sum_{t=1}^{T}\mas(G_t)$ (or
on $\sum_{t=1}^{T}\alpha(G_t)$, in the symmetric case). Unfortunately, in the
worst case, computing the maximal acyclic subgraph or the independence number
of a given graph is NP-hard, so implementing such algorithms is not
\emph{always} computationally tractable.\footnote { \cite{MS11} proposed a
generic mechanism to circumvent this, but the justification has a flaw which
is not clear how to fix. } However, it is easy to see that the algorithm is
robust to approximate computation of this value ---misspecifying the average
independence number $\frac{1}{T}\sum_{t=1}^{T}\alpha(G_t)$ by a factor of $v$
entails an additional $\sqrt{v}$ factor in the bound. Thus, one might use
standard heuristics resulting in a reasonable approximation of the
independence number. Although computing the independence number is also
NP-hard to approximate, it is unlikely for intricate graphs with
hard-to-approximate independence numbers to appear in relevant applications.
Moreover, by setting the approximation to be either $K$ or $1$, we trivially
obtain an approximation factor of at most either $K$ or
$\frac{1}{T}\sum_{t=1}^{T}\alpha(G_t)$. The former leads to a
$\widetilde{\Ocal}(\sqrt{KT})$ regret bound similar to the standard bandits
setting, while the latter leads to a
$\widetilde{\Ocal}\left(\frac{1}{T}\sum_{t=1}^{T}\alpha(G_T)\sqrt{T}\right)$
regret bound, which is better than the regret for the bandits setting if the
average independence number is less than $\sqrt{K}$. In contrast, this
computational issue does not show up in Exp3-DOM, whose tuning relies only on
efficiently-computable quantities.

\section{Conclusions and Open Questions}\label{s:conc}
In this paper we investigated online prediction problems in partial
information regimes that interpolate between the classical bandit and expert
settings. We provided algorithms, as well as upper and lower bounds on the
attainable regret, with a non-trivial dependence on the information feedback
structure. In particular, we have shown a number of results characterizing
prediction performance in terms of: the structure of the feedback system, the
amount of information available before prediction, and the nature
(adversarial or fully random) of the process generating the feedback system.

There are many open questions that warrant further study, some of which are
briefly mentioned below:
\begin{enumerate}
\item It would be interesting to study adaptations of our results to the
    case when the feedback system $\{S_{i,t}\}_{i\in V}$ may depend on the
    loss $\ell_{I_t,t}$ of player's action $I_t$. Note that this would
    prevent a direct construction of an unbiased estimator for unobserved
    losses, which many worst-case bandit algorithms (including ours ---see
    the appendix) hinge upon.
%
\item The upper bound contained in Theorem~\ref{c:ndag}, expressed in terms
    of $\mas(\cdot)$, is almost certainly suboptimal, even in the
    uninformed setting, and it would be nice to see if more adequate graph
    complexity measures can be used instead.
\item Our lower bound in Theorem ~\ref{th:lowerbound}
refers to a constant graph sequence. We would like to provide a more
complete characterization applying to sequences of adversarially-generated
graphs $G_1, G_2, \ldots, G_T$ in terms of sequences of their corresponding
independence numbers $\alpha(G_1), \alpha(G_2), \ldots, \alpha(G_T)$ (or
variants thereof), in both the uninformed and the informed settings.
Moreover, the adversary strategy achieving our lower bound is
computationally hard to implement in the worst case (the adversary needs to
identify the largest independent set in a given graph). What is the
achievable regret if the adversary is assumed to be computationally
bounded?
\item The information feedback models we used are natural and simple. They
    assume that the action at a give time period only affects rewards and
    observations for that period. In some settings, the reward observation
    may be delayed. In such settings, the action taken at a given stage may
    affect what is observed in subsequent stages. We leave the issue of
    modelling and analyzing such setting to future work.
\item Finally, we would like to see what is the achievable performance in
    the special case of stochastic rewards, which are assumed to be drawn
    i.i.d.\ from some unknown distributions. This was recently considered
    in \cite{CaKvLeBh12}, with results depending on the graph clique
    structure. However, the tightness of these results remains to be
    ascertained.
%
%
\end{enumerate}

\subsubsection*{Acknowledgments}

NA was supported in part by a USA-Israeli BSF grant, by an ISF grant, by the
Israeli I-Core program and by the Oswald Veblen Fund. NCB acknowledges
partial support by MIUR (project ARS TechnoMedia, PRIN 2010-2011, grant no.\
2010N5K7EB\_003). SM was supported in part by the European Community's
Seventh Framework Programme (FP7/2007-2013) under grant agreement 306638
(SUPREL). YM was supported in part by a grant from the Israel Science
Foundation, a grant from the United States-Israel Binational Science
Foundation (BSF), a grant by Israel Ministry of Science and Technology and
the Israeli Centers of Research Excellence (I-CORE) program (Center No.\
4/11). OS was supported in part by a grant from the Israel Science Foundation
(No. 425/13) and a Marie-Curie Career Integration Grant.


\bibliographystyle{plain}
\bibliography{mab,nic}


\appendix

\section{Technical lemmas and proofs from Section~\ref{s:symm}}\label{s:appendixexp3set}
This section contains the proofs of all technical results occurring in
Section \ref{s:symm}, along with ancillary graph-theoretic lemmas. Throughout
this appendix, $\E_t[\cdot]$ is a shorthand for $ \E\bigl[\cdot\mid
I_1,\dots,I_{t-1}\bigr]$. Also, for ease of exposition, we implicitly first
condition on the history, i.e., $I_1, I_2, \ldots, I_{t-1}$, and later take
an expectation with respect to that history. This implies that, given that
conditioning, we can treat random variables such as $p_{i,t}$ as constants,
and we can later take an expectation over history so as to remove the
conditioning.

\subsection{Proof of Fact \ref{l:bad}}
%
%
Using standard properties of geometric sums, one can immediately see that
\[
\sum_{i=1}^{K}\frac{p_i}{\sum_{j=i}^K p_j} = \sum_{i=1}^{K-1}
\frac{2^{-i}}{2^{-i+1}} + \frac{2^{-K+1}}{2^{-K+1}}= \frac{K-1}{2} +1= \frac{K+1}{2}~,
\]
hence the claimed result.

\subsection{Proof of Lemma \ref{thm:noexp}}
\label{s:appendix:thm:noexp}
Following the proof of Exp3~\cite{AuerCeFrSc02}, we have
\begin{align*}
\frac{W_{t+1}}{W_t}
&= \sum_{i \in V} \frac{w_{i,t+1}}{W_t}\\
&= \sum_{i \in V} \frac{w_{i,t}\,\exp(-\eta\,\hloss_{i,t})}{W_t}\\
&= \sum_{i \in V} p_{i,t}\,\exp(-\eta\,\hloss_{i,t})\\
&\leq \sum_{i \in V} p_{i,t}\,\left(1 - \eta\hloss_{i,t} + \frac{1}{2}\,\eta^2(\hloss_{i,t})^2\right) \quad \text{using $e^{-x} \leq 1-x+x^2/2$ for all $x \ge 0$}\\
&\leq 1 - \eta\,\sum_{i \in V} p_{i,t}\hloss_{i,t} + \frac{\eta^2}{2}\,\sum_{i \in V} p_{i,t}(\hloss_{i,t})^2~.
\end{align*}
Taking logs, using $\ln(1-x) \le -x$ for all $x \ge 0$, and summing over $t =
1, \ldots, T$ yields
\[
\ln \frac{W_{T+1}}{W_1} \leq -\eta\,\sum_{t=1}^T \sum_{i \in V} p_{i,t}\hloss_{i,t} +
\frac{\eta^2}{2}\,\sum_{t=1}^T \sum_{i \in V} p_{i,t}(\hloss_{i,t})^2~.
\]
Moreover, for any fixed comparison action $k$, we also have
\[
\ln \frac{W_{T+1}}{W_1} \geq \ln \frac{w_{k,T+1}}{W_1} = -\eta\,\sum_{t=1}^T \hloss_{k,t} - \ln K~.
\]
Putting together and rearranging gives
\begin{equation}\label{e:eq1}
    \sum_{t=1}^T \sum_{i \in V} p_{i,t}\hloss_{i,t}
\le
    \sum_{t=1}^T \hloss_{k,t} + \frac{\ln K}{\eta}
    + \frac{\eta}{2}\,\sum_{t=1}^T \sum_{i \in V} p_{i,t}(\hloss_{i,t})^2~.
\end{equation}
Note that, for all $i \in V$,
\[
\E_t[\hloss_{i,t}] = \sum_{j\,:\, i \in S_{j,t}}
p_{j,t}\,\frac{\loss_{i,t}}{q_{i,t}}
                = \sum_{j \,:\, j \reach{t} i} p_{j,t}\,\frac{\loss_{i,t}}{q_{i,t}}
                = \frac{\loss_{i,t}}{q_{i,t}}\sum_{j \,:\, j \reach{t} i} p_{j,t}
                = \loss_{i,t}~.
\]
Moreover,
\[
\E_t\bigl[(\hloss_{i,t})^2\bigr] = \sum_{j\,:\, i \in S_{j,t}}
p_{j,t}\,\frac{\loss^2_{i,t}}{q^2_{i,t}}
                = \frac{\loss^2_{i,t}}{q^2_{i,t}}\sum_{j \,:\, j \reach{t} i} p_{j,t}
                \leq \frac{1}{q^2_{i,t}}\sum_{j \,:\, j \reach{t} i} p_{j,t}
                = \frac{1}{q_{i,t}}~.
\]
Hence, taking expectations $\E_t$ on both sides of~(\ref{e:eq1}), and
recalling the definition of $Q_t$, we can write
\begin{equation}\label{e:conditionalregret}
    \sum_{t=1}^T \sum_{i \in V} p_{i,t}\,\loss_{i,t}
\le
    \sum_{t=1}^T \loss_{k,t} + \frac{\ln K}{\eta} + \frac{\eta}{2}\,\sum_{t=1}^T Q_t~.
\end{equation}
Finally, taking expectations over history to remove conditioning gives
\begin{equation*}
\E\bigl[L_{A,T} - L_{k,T}\bigr]
\le
    \frac{\ln K}{\eta} + \frac{\eta}{2}\sum_{t=1}^T \E[Q_t]
\end{equation*}
as claimed. \hfill\qed

\subsection{Proof of Theorem \ref{c:ndag}}
\label{s:appendix:c:ndag}
We first need the following lemma.


\begin{lemma}\label{lemma:nDGA}
Let $G = (V,D)$ be a directed graph with vertex set $V = \{1,\ldots,K\}$, and
arc set $D$.
Then, for any distribution $p$ over $V$ we have,
\[
\sum_{i=1}^K \frac{p_i}{p_i+ \sum_{j\,:\,j \reach{}i}
p_j}
\leq \mas(G)~.
\]
\end{lemma}
\begin{proof}
We show that there is a subset of vertices $V'$ such that the induced graph
is acyclic and $|V'|\geq \sum_{i=1}^K \frac{p_i}{p_i+ \sum_{j \in N_i^-}\
p_j}$. Let $N_i^-$ be the in-neighborhood of node $i$, i.e., the set of nodes
$j$ such that $(j,i) \in D$.

We prove the lemma by adding elements to an initially empty set $V'$.
Let
\[
\Phi_0=\sum_{i=1}^K \frac{p_i}{p_i+ \sum_{j\,:\,j \reach{}i}
p_j},
\]
and let $i_1$ be the vertex which minimizes $p_{i}+ \sum_{j \in N_{i}^-}\
p_j$ over $i \in V$.
We now delete $i_1$ from the graph, along with all its incoming neighbors
(set $N_{i_1}^-$), and all edges which are incident (both departing and
incoming) to these nodes, and then iterating on the remaining graph. Let
$N_{i,1}^-$ be the in-neighborhoods of the graph after the first step.
The contribution of all the deleted vertices to $\Phi_0$ is
\[
\sum_{r\in N_{i_1}^-\cup \{i_1\}} \frac{p_r}{p_r+ \sum_{j \in
N_r^-}\ p_j} \leq \sum_{r\in N_{i_1}^-\cup \{i_1\}}
\frac{p_r}{p_{i_1}+ \sum_{j \in N_{i_1}^-}\ p_j}=1\,,
\]
where the inequality follows from the minimality of $i_1$.

Let $V' \leftarrow V'\cup\{i_1\}$, and  $V_1= V\setminus(N_{i_1}^-\cup
\{i_1\})$. Then, from the first step we have
\[
\Phi_1=\sum_{i\in V_1} \frac{p_i}{p_i+ \sum_{j \in N_{i,1}^-}\ p_j}
\geq
\sum_{i\in V_1} \frac{p_i}{p_i+ \sum_{j \in N_{i}^-}\ p_j}
\geq
\Phi_0 -1~.
\]
We apply the very same argument to $\Phi_1$ with node $i_2$ (minimizing $p_i+
\sum_{j \in N_{i,1}^-}\ p_j$ over $i \in V_1$), to $\Phi_2$ with node $i_3$,
\ldots, to $\Phi_{s-1}$ with node $i_s$, up until $\Phi_s = 0$, i.e., until
no nodes are left in the reduced graph. This gives $\Phi_0 \leq s = |V'|$,
where $V' = \{i_1, i_2, \ldots, i_s\}$. Moreover, since in each step $r = 1,
\ldots, s$ we remove all remaining arcs incoming to $i_r$, the graph induced
by set $V'$ cannot contain cycles.
%
%
\end{proof}

The claim of Theorem~\ref{c:ndag} follows from a direct combination of
Lemma~\ref{thm:noexp} with Lemma~\ref{lemma:nDGA}.

\subsection{Proof of Theorem \ref{th:lowerbound}}\label{subsec:prooflowerbound}

The proof uses a variant of the standard multi-armed bandit lower bound
\cite{cbl06}. The intuition is that when we have $\alpha(G)$ non-adjacent
nodes, the problem reduces to an instance of the standard multi-armed bandit
(where information beyond the loss of the action choses is observed) on
$\alpha(G)$ actions.

By Yao's minimax principle, in order to establish the lower bound, it is
enough to demonstrate some probabilistic adversary strategy, on which the
expected regret of any \emph{deterministic} algorithm $A$ is bounded from
below by $C\sqrt{\alpha(G) T}$ for some constant $C$.

Specifically, suppose without loss of generality that we number the nfiodes
in some largest independent set of $G$ by $1,2,\ldots,\alpha(G)$, and all the
other nodes in the graph by $\alpha(G)+1,\ldots,|V|$. Let $\epsilon$ be a
parameter to be determined later, and consider the following joint
distribution over stochastic loss sequences:
\begin{itemize}
  \item Let $Z$ be uniformly distributed on $1,2,\ldots,\alpha(G)$;
  \item Conditioned on $Z=i$, each loss $\ell_{j,t}$ is independent
      Bernoulli with parameter $1/2$ if $j\neq i$ and $j<\alpha(G)$,
      independent Bernoulli with parameter $1/2-\epsilon$ if $j=i$, and is
      $1$ with probability $1$, otherwise.
\end{itemize}
For each $i=1\ldots \alpha(G)$, let $T_i$ be the number of times the node $i$
was chosen by the algorithm after $T$ rounds. Also, let $T_{\Delta}$ denote
the number of times some node whose index is larger than $\alpha(G)$ is
chosen after $T$ rounds. Finally, let $\E_i$ denote expectation conditioned
on $Z=i$, and ${\Pr}_i$ denote the probability over loss sequences
conditioned on $Z=i$. We have
\begin{align*}
\max_{k\in V}\E[L_{A,T}-L_{k,T}] &= \frac{1}{\alpha(G)}\sum_{i=1}^{\alpha(G)}
\E_i\left[L_{A,T}-\left(\frac{1}{2}-\epsilon\right)T\right]\\
&= \frac{1}{\alpha(G)}\sum_{i=1}^{\alpha(G)}\E_i\left[\sum_{j\in \{1\ldots \alpha(G)\}\setminus i}\frac{1}{2}T_j+\left(\frac{1}{2}-\epsilon\right)T_i+T_{\Delta}
-\left(\frac{1}{2}-\epsilon\right)T\right]\\
&= \frac{1}{\alpha(G)}\sum_{i=1}^{\alpha(G)}\E_i\left[\frac{1}{2}\sum_{j=1}^{\alpha(G)}T_j
+\frac{1}{2}T_{\Delta}+\frac{1}{2}T_{\Delta}-\epsilon T_i - \left(\frac{1}{2}-\epsilon\right)T\right].
\end{align*}
Since $\sum_{j=1}^{\alpha(G)}T_j+T_{\Delta}=T$, this expression equals
\begin{equation}\label{eq:regretl}
\frac{1}{\alpha(G)}\sum_{i=1}^{\alpha(G)}\E_i\left[\frac{1}{2}T_{\Delta}+\epsilon(T-T_i)\right]
~\geq~ \epsilon\left(T-\frac{1}{\alpha(G)}\sum_{i=1}^{\alpha(G)}\E_i[T_i]\right).
\end{equation}
Now, consider another distribution ${\Pr}_{0}$ over the loss sequence, which
corresponds to the distribution above but with $\epsilon=0$ (namely, all
nodes $1,\ldots,\alpha(G)$ have losses which are $\pm 1$ independently and
with equal probability, and all nodes whose index is larger than $\alpha(G)$
have losses of $1$), and denote by $\E_0$ the corresponding expectation. We
upper bound the difference between $\E_i[T_i]$ and $\E_0[T_i]$, using
information theoretic arguments.
Let $\lambda_t$ be the collection of loss values observed at round $t$, and
$\lambda^t=(\lambda_1,\ldots,\lambda_{t})$. Note that since the algorithm is
deterministic, $\lambda^{t-1}$ determines the algorithm's choice of action
$I_t$ at each round $t$, and hence $T_i$ is determined by $\lambda^T$, and
thus $\E_0[T_i \mid \lambda^T]=\E_i[T_i \mid \lambda^T]$. We have
\begin{align*}
  \E_i[T_i]-\E_0[T_i] &= \sum_{\lambda^T}{\Pr}_i(\lambda^T)\E_i[T_i \mid \lambda^T]-\sum_{\lambda^T}{\Pr}_{0}(\lambda^T)\E_0[T_i \mid \lambda^T]\\
  &=   \sum_{\lambda^T}{\Pr}_i(\lambda^T)\E_i[T_i \mid \lambda^T]-\sum_{\lambda^T}{\Pr}_{0}(\lambda^T)\E_i[T_i|\lambda^T]\\
  &\leq \sum_{\lambda^T\,:\,{\Pr}_i(\lambda^T)>{\Pr}_{0}(\lambda^T)}\left({\Pr}_i(\lambda^T)-{\Pr}_{0}(\lambda^T)\right)\E_i[T_i \mid \lambda^T]\\
  &\leq T \sum_{\lambda^T\,:\,{\Pr}_i(\lambda^T)>{\Pr}_{0}(\lambda^T)}\left({\Pr}_i(\lambda^T)-{\Pr}_{0}(\lambda^T)\right)~.
\end{align*}
Using Pinsker's inequality, this is at most
\[
T\sqrt{\frac{1}{2}D_{\mathrm{kl}}({\Pr}_{0}(\lambda^T)\,\|\,{\Pr}_i(\lambda^T))}
\]
where $D_{\mathrm{kl}}$ is the Kullback-Leibler divergence (or relative
entropy) between the distributions ${\Pr}_i$ and ${\Pr}_{0}$. Using the chain
rule for relative entropy, this equals
\[
T\sqrt{\frac{1}{2}\sum_{t=1}^{T}\sum_{\lambda^{t-1}}{\Pr}_{0}(\lambda^{t-1})D_{\mathrm{kl}}
\bigl({\Pr}_{0}(\lambda_t|\lambda^{t-1})\,\|\,{\Pr}_i(\lambda_t|\lambda^{t-1})\bigr)}~.
\]
Let us consider any single relative entropy term above. Recall that
$\lambda^{t-1}$ determines the node $I_t$ picked at round $t$. If this node
is not $i$ or adjacent to $i$, then $\lambda_t$ is going to have the same
distribution under both ${\Pr}_i$ and ${\Pr}_{0}$, and the relative entropy
is zero. Otherwise, the coordinate of $\lambda_t$ corresponding to node $i$
(and that coordinate only) will have a different distribution: Bernoulli with
parameter $\frac{1}{2}-\epsilon$ under ${\Pr}_i$, and Bernoulli with
parameter $\frac{1}{2}$ under ${\Pr}_{0}$. The relative entropy term in this
case is easily shown to be $-\frac{1}{2}\log(1-4\epsilon^2) \leq
8\log(4/3)\,\epsilon^2$. Therefore, letting $S_{I_t}$ denote the feedback set
at time $t$, we can upper bound the above by
\begin{eqnarray}
T\sqrt{\frac{1}{2}\sum_{t=1}^{T}{\Pr}_{0}(i\in S_{I_t})(8\log(4/3)\epsilon^2)}
&=&
2T\epsilon \sqrt{\log\left(\frac{4}{3}\right)\E_0\bigl[|\{t:i\in S_{I_t}\}|\bigr]}\nonumber\\
&\leq&
2T\epsilon \sqrt{\log\left(\frac{4}{3}\right)\E_0\left[T_i+T_{\Delta}\right]}~.\label{eq:dkll}
\end{eqnarray}
We now claim that we can assume $\E_0[T_{\Delta}]\leq 0.08\sqrt{\alpha(G)
T}$. To see why, note that if $\E_0[T_{\Delta}]> 0.08\sqrt{\alpha(G) T}$,
then the expected regret under $\E_0$ would have been at least
\begin{eqnarray*}
\max_{k\in V}\E_0[L_{A,T}-L_{k,T}]
&=&  \E_0\left[T_\Delta+\frac{1}{2}\sum_{j=1}^{\alpha(G)}T_j\right]-\frac{1}{2}T\\
&=&  \E_0\left[\frac{1}{2}T_\Delta+\frac{1}{2}\left(T_{\Delta}+\sum_{j=1}^{\alpha(G)}T_j\right)\right]-\frac{1}{2}T\\
&=& \E_0\left[\frac{1}{2}T_{\Delta}+\frac{1}{2}T\right]-\frac{1}{2}T\\
&=& \frac{1}{2}\E_0[T_{\Delta}]\\
&>& 0.04\sqrt{\alpha(G)T}~.
\end{eqnarray*}
So for the adversary strategy defined by the distribution ${\Pr}_{0}$, we
would get an expected regret lower bound as required. Thus, it only remains
to treat the case where $\E_0[T_{\Delta}]\leq 0.08\sqrt{\alpha(G) T}$.
Plugging in this upper bound into \eqref{eq:dkll}, we get overall that
\[
\E_i[T_i]-\E_0[T_i] \leq 2T\epsilon \sqrt{\log\left(\frac{4}{3}\right)\E_0\left[T_i+0.08\sqrt{\alpha(G) T}\right]}~.
\]
Therefore, the expected regret lower bound in \eqref{eq:regretl} is at least
\begin{align*}
&\epsilon\left(T-\frac{1}{\alpha(G)}\sum_{i=1}^{\alpha(G)}\E_0[T_i]
-\frac{1}{\alpha(G)}\sum_{i=1}^{\alpha(G)}2T\epsilon\sqrt{\log\left(\frac{4}{3}\right)\E_0\left[T_i+0.08\sqrt{\alpha(G) T}\right]}\right)\\
&\geq~
\epsilon\left(T-\frac{T}{\alpha(G)}
-2T\epsilon\sqrt{\log\left(\frac{4}{3}\right)\frac{1}{\alpha(G)}\sum_{i=1}^{\alpha(G)}\E_0\left[T_i+0.08\sqrt{\alpha(G) T}\right]}\right)\\
&\geq~
\epsilon T\left(1-\frac{1}{\alpha(G)}
-2\epsilon\sqrt{\log\left(\frac{4}{3}\right)\left(\frac{T}{\alpha(G)}+0.08\sqrt{\alpha(G) T}\right)}\right).
\end{align*}
Since $\alpha(G)>1$, we have $1-\frac{1}{\alpha(G)} \geq \frac{1}{2}$, and
since $T\geq 0.0064\alpha^3(G)$, we have $0.08\sqrt{\alpha(G)T}\leq
\frac{T}{\alpha(G)}$. Overall, we can lower bound the expression above by
\[
\epsilon T\left(\frac{1}{2}
-2\epsilon\sqrt{2\log\left(\frac{4}{3}\right)\frac{T}{\alpha(G)}}\right).
\]
Picking $\epsilon = \frac{1}{8\sqrt{2\log(4/3)T/\alpha(G)}}$, the expression
above is
\[
\frac{T}{8\sqrt{2\log\left(\frac{4}{3}\right)\frac{T}{\alpha(G)}}}\frac{1}{4} \geq 0.04 \sqrt{\alpha(G)T}~.
\]
This constitutes a lower bound on the expected regret, from which the result
follows.

\subsection{Proof of Theorem \ref{thm:random_er}}\label{sa:erdos}
Fix round $t$, and let $G = (V,D)$ be the Erd\H{o}s-Renyi random graph
generated at time $t$, $N_i^-$ be the in-neighborhood of node $i$, i.e., the
set of nodes $j$ such that $(j,i) \in D$, and denote by $d^-_i$ the indegree
of $i$. We need the following lemmas.

\begin{lemma}\label{cl:1}
Fix a directed graph $G=(V,D)$. Let $p_1, \ldots, p_K$ be an arbitrary
probability distribution defined over $V$, $f: V\rightarrow V$ be an
arbitrary permutation of $V$, and ${\E_f}$ denote the expectation w.r.t.\ a
random permutation $f$. Then, for any $i \in V$, we have
\[
\E_f\left[\frac{p_{f(i)}}
             {p_{f(i)} + \sum_{j\,:\,j \reach{} i}
             p_{f(j)}}\right]
=
\frac{1}{1+d^-_i}~.
\]
\end{lemma}
\begin{proof}
Consider selecting a subset $S\subset V$ of $1+d^-_i$ nodes. We consider the
contribution to the expectation when $S=N^-_{f(i)}\cup\{f(i)\}$. Since there
are $K(K-1)\cdots(K-d^-_i+1)$ terms (out of $K!$) contributing to the
expectation, we can write
\begin{eqnarray*}
\E_f\left[\frac{p_{f(i)}}{p_{f(i)} + \sum_{j\,:\, f(j) \in N^-_{f(i)}} p_{f(j)}}\right]
&=&
\frac{1}{\binom{K}{d^-_i}}
\sum_{S\subset V, |S|=d^-_i} \frac{1}{1+d^-_i}\sum_{i\in S} \frac{p_i}{p_i+\sum_{j\in S,j\neq i} p_j}\\
&=&
\frac{1}{\binom{K}{d^-_i}}
\sum_{S\subset V, |S|=d^-_i} \frac{1}{1+d^-_i}\\
&=& \frac{1}{1+d^-_i}~.
\end{eqnarray*}
\end{proof}

\begin{lemma}\label{cl:2}
Let $p_1, \ldots, p_K$ be an arbitrary probability distribution defined over
$V$, and ${\E}$ denote the expectation w.r.t.\ the Erd\H{o}s-Renyi random
draw of arcs at time $t$. Then, for any fixed $i \in V$, we have
\[
\E\left[\frac{p_i}{p_i + \sum_{j\,:\,j \reach{t}i} p_j}\right]
=
\frac{1}{rK}\left(1-(1-r)^K\right)~.
\]
\end{lemma}
\begin{proof}
For the given $i \in V$ and time $t$, consider the Bernoulli random variables
$X_{j}, j\in V\setminus\{i\}$, and denote by $\E_{j\,:\,j\neq i}$ the
expectation w.r.t.\ all of them. We symmetrize $\E\left[\frac{p_i}{p_i +
\sum_{j\,:\,j \reach{t}i} p_j}\right]$ by means of a random permutation $f$,
as in Lemma \ref{cl:1}. We can write
\begin{eqnarray*}
\E\left[\frac{p_i}{p_i + \sum_{j\,:\,j \reach{t}i} p_j}\right]
&= &
\E_{j\,:\,j\neq i}\left[\frac{p_i}{p_i + \sum_{j\,:\,j\neq i} X_j p_j}\right]\\
&= &
\E_{j\,:\,j\neq i} \E_f \left[\frac{p_{f(i)}}{p_{f(i)} + \sum_{j\,:\,j\neq i} X_{f(j)} p_{f(j)}}\right]
        \qquad ({\mbox{by symmetry}})\\
&= &
\E_{j\,:\,j\neq i} \left[\frac{1}{1 + \sum_{j\,:\,j\neq i} X_j}\right]
        \qquad ({\mbox{from Lemma \ref{cl:1}}})\\
&=&
\sum_{i=0}^{K-1} \binom{K-1}{i} r^i (1-r)^{K-1-i} \frac{1}{i+1}\\
&=&
\frac{1}{rK} \sum_{i=0}^{K-1} \binom{K}{i+1} r^{i+1}(1-r)^{K-1-i} \\
&=&
\frac{1}{rK} \left( 1 - (1-r)^K\right)~.
\end{eqnarray*}
\end{proof}

At this point, we follow the proof of Lemma~\ref{thm:noexp} up
until~(\ref{e:conditionalregret}). We take an expectation $\E_{G_1, \ldots,
G_T}$ w.r.t.\ the randomness in generating the sequence of graphs $G_1,
\ldots, G_T$. This yields
\[
    \sum_{t=1}^T \E_{G_1, \ldots, G_T}\left[\sum_{i \in V} p_{i,t}\,\loss_{i,t}\right]
\le
    \sum_{t=1}^T \loss_{k,t} + \frac{\ln K}{\eta} + \frac{\eta}{2}\,\sum_{t=1}^T \E_{G_1, \ldots, G_T}\left[Q_t\right]~.
\]
We use Lemma~\ref{cl:2} to upper bound $\E_{G_1, \ldots,
G_T}\left[Q_t\right]$ by $\frac{1}{r} \left( 1 - (1-r)^K\right)$, and take
the outer expectation to remove conditioning, as in
the proof of Lemma~\ref{thm:noexp}. This concludes the proof. 

\section{Technical lemmas and proofs from Section~\ref{ss:expdom}}\label{s:appendixexp3dom}

Again, throughout this appendix, $\E_t[\,\cdot\,]$ is a shorthand for the
conditional expectation $\E_t[\,\cdot\,|\, I_1, I_2, \ldots, I_{t-1}]$.
Moreover, as we did in Appendix \ref{s:appendixexp3set}, in round $t$ we
first condition on the history $I_1, I_2, \ldots, I_{t-1}$, and then take an
outer expectation with respect to that history.


\subsection{Proof of Lemma \ref{thm:alg}}\label{sa:lemma_exp3_dom}
We start to bound the contribution to the overall regret of an instance
indexed by $b$. When clear from the context, we remove the superscript $b$
from $\gammab$, $w^{(b)}_{i,t}$, $p^{(b)}_{i,t}$, and other related
quantities. For any $t\in\Tb$ we have
\begin{align*}
    \frac{W_{t+1}}{W_t}
&=
    \sum_{i \in V} \frac{w_{i,t+1}}{W_t}
\\&=
    \sum_{i \in V} \frac{w_{i,t}}{W_t}\,\exp\bigl(-(\gamma/2^b)\,\hloss_{i,t}\bigr)
\\&=
    \sum_{i \in R_t} \frac{p_{i,t}-\gamma/|R_t|}{1-\gamma}\,\exp\bigl(-(\gamma/2^b)\,\hloss_{i,t}\bigr) + \sum_{i \not\in R_t} \frac{p_{i,t}}{1-\gamma}\,\exp\bigl(-(\gamma/2^b)\,\hloss_{i,t}\bigr)
\\ &\le
    \sum_{i \in R_t} \frac{p_{i,t}-\gamma/|R_t|}{1-\gamma}\,
    \left(1 - \frac{\gamma}{2^b}\hloss_{i,t} + \frac{1}{2}\left(\frac{\gamma}{2^b}\hloss_{i,t}\right)^2\right)
    + \sum_{i \not\in R_t} \frac{p_{i,t}}{1-\gamma}\,\left(1 - \frac{\gamma}{2^b}\hloss_{i,t} + \frac{1}{2}\left(\frac{\gamma}{2^b}\hloss_{i,t}\right)^2\right)\\
& \text{(using $e^{-x} \leq 1-x+x^2/2$ for all $x \ge 0$)}
\\ &\le
    1 - \frac{\gamma/2^b}{1-\gamma}\sum_{i \in V} p_{i,t}\hloss_{i,t}
    + \frac{\gamma^2/2^b}{1-\gamma}\sum_{i \in R_t} \frac{\hloss_{i,t}}{|R_t|}
    + \frac{1}{2}\frac{(\gamma/2^b)^2}{1-\gamma}\sum_{i \in V} p_{i,t}\bigl(\hloss_{i,t}\bigr)^2~.
\end{align*}
%
%
Taking logs, upper bounding, and summing over $t \in \Tb$ yields
\[
    \ln\frac{W_{|\Tb|+1}}{W_1}
\le
    - \frac{\gamma/2^b}{1-\gamma}\sum_{t \in \Tb} \sum_{i \in V} p_{i,t}\hloss_{i,t}
    + \frac{\gamma^2/2^b}{1-\gamma}\sum_{t \in \Tb} \sum_{i \in R_t} \frac{\hloss_{i,t}}{|R_t|}
    + \frac{1}{2}\frac{(\gamma/2^b)^2}{1-\gamma} \sum_{t \in \Tb}\sum_{i \in V} p_{i,t}\bigl(\hloss_{i,t}\bigr)^2~.
\]
Moreover, for any fixed comparison action $k$, we also have
\[
    \ln\frac{W_{|\Tb|+1}}{W_1}
\ge
    \ln\frac{w_{k,|\Tb|+1}}{W_1} = -\frac{\gamma}{2^b}\sum_{t \in \Tb} \hloss_{k,t} - \ln K~.
\]
Putting together, rearranging, and using $1-\gamma \le 1$ gives
\[
    \sum_{t \in \Tb} \sum_{i \in V} p_{i,t}\hloss_{i,t}
\le
    \sum_{t \in \Tb} \hloss_{k,t} + \frac{2^b\ln K}{\gamma}
    + \gamma\sum_{t \in \Tb} \sum_{i \in R_t} \frac{\hloss_{i,t}}{|R_t|}
    + \frac{\gamma}{2^{b+1}} \sum_{t \in \Tb}\sum_{i \in V} p_{i,t}\bigl(\hloss_{i,t}\bigr)^2~.
\]
Reintroducing the notation $\gammab$ and summing over
$b=0,1,\dots,\lfloor\log_2 K\rfloor$ gives
\begin{equation}\label{e:eq2}
    \sum_{t=1}^T \left( \sum_{i \in V} p^{(b_t)}_{i,t}\hloss^{(b_t)}_{i,t} - \hloss_{k,t} \right)
\le
    \sum_{b=0}^{\lfloor\log_2 K\rfloor}\frac{2^b\ln K}{\gammab}
    + \sum_{t=1}^T \sum_{i \in R_t} \frac{\gamma^{(b_t)}\hloss^{(b_t)}_{i,t}}{|R_t|}
    + \sum_{t=1}^T \frac{\gamma^{(b_t)}}{2^{b_t+1}} \sum_{i \in V}  p^{(b_t)}_{i,t}\bigl(\hloss^{(b_t)}_{i,t}\bigr)^2~.
\end{equation}
Now, similarly to the proof of Lemma~\ref{thm:noexp}, we have that $
    \E_t\bigl[\hloss^{(b_t)}_{i,t}\bigr] = \loss_{i,t}
$ and $
    \E_t\bigl[(\hloss^{(b_t)}_{i,t})^2\bigr] \leq \frac{1}{q^{(b_t)}_{i,t}}
$ for any $i$ and $t$. Hence, taking expectations $\E_t$ on both sides of
(\ref{e:eq2}) and recalling the definition of $Q^{(b)}_t$ gives
\begin{equation}\label{e:eq3}
    \sum_{t=1}^T \left( \sum_{i \in V} p^{(b_t)}_{i,t}\ell_{i,t} - \ell_{k,t} \right)
\le
    \sum_{b=0}^{\lfloor\log_2 K\rfloor}\frac{2^b\ln K}{\gammab}
    + \sum_{t=1}^T \sum_{i \in R_t} \frac{\gamma^{(b_t)}\ell_{i,t}}{|R_t|}
    + \sum_{t=1}^T \frac{\gamma^{(b_t)}}{2^{b_t+1}} Q^{(b_t)}_t~.
\end{equation}
Moreover,
\[
\sum_{t=1}^T \sum_{i \in R_t} \frac{\gamma^{(b_t)}\ell_{i,t}}{|R_t|}
\leq
\sum_{t=1}^T \sum_{i \in R_t} \frac{\gamma^{(b_t)}}{|R_t|}
=
\sum_{t=1}^T \gamma^{(b_t)}
=
\sum_{b=0}^{\lfloor\log_2 K\rfloor} \gammab|\Tb|
\]
and
\[
\sum_{t=1}^T \frac{\gamma^{(b_t)}}{2^{b_t+1}} Q^{(b_t)}_t
=
\sum_{b=0}^{\lfloor\log_2 K\rfloor} \frac{\gamma^{(b)}}{2^{b+1}} \sum_{t\in T^{(b)}} Q^{(b)}_t ~.
\]
Hence, substituting back into~(\ref{e:eq3}), taking outer expectations on
both sides and recalling that $\Tb$ is a random variable (since the adversary
adaptively decides which steps $t$ fall into $\Tb$), we get
\begin{align}
\nonumber
    \E\bigl[L_{A,T} - L_{k,T}\bigr]
& \le
    \sum_{b=0}^{\lfloor\log_2 K\rfloor}\E\left[\frac{2^b\ln K}{\gammab} + \gammab|\Tb|
    + \frac{\gammab}{2^{b+1}}\sum_{t \in \Tb} Q^{(b)}_t\right]
\\ &=
\label{eq:doubling}
    \sum_{b=0}^{\lfloor\log_2 K\rfloor} \left( \frac{2^b\ln K}{\gammab}
    +  \gammab\E\left[\sum_{t \in \Tb} \left(1 + \frac{Q^{(b)}_t}{2^{b+1}}\right)\right]\right)~.
\end{align}
%
%
This establishes~(\ref{eq:gammabfixed}).

In order to prove inequality~(\ref{eq:doublingtrick}), we need to tune each
$\gammab$ separately. However, a good choice of $\gammab$ depends on the
unknown random quantity
\[
    \overline{Q}^{(b)} = \sum_{t \in \Tb} \left(1 + \frac{Q^{(b)}_t}{2^{b+1}}\right)~.
\]
To overcome this problem, we slightly modify Exp3-DOM by applying a doubling
trick\footnote { The pseudo-code for the variant of Exp3-DOM using such a
doubling trick is not displayed here, since it is by now a folklore
technique.
} to guess $\overline{Q}^{(b)}$ for each $b$. Specifically, for each $b = 0,
1, \ldots, \lfloor \log_2 K \rfloor$, we use a sequence $\gammab_r =
\sqrt{(2^b\ln K)/2^r}$, for $r=0,1,\dots$. We initially run the algorithm
with $\gammab_0$. Whenever the algorithm is running with $\gammab_r$ and
observes that $\sum_s\overline{Q}^{(b)}_s > 2^r$, where the sum is over all
$s$ so far in $\Tb$,\footnote { Notice that $\sum_s\overline{Q}^{(b)}_s$ is
an observable quantity. } then we restart the algorithm with $\gammab_{r+1}$.
Because the contribution of instance $b$ to~(\ref{eq:doubling}) is
\[
    \frac{2^b\ln K}{\gammab} + \gammab\sum_{t \in \Tb} \left(1 + \frac{Q^{(b)}_t}{2^{b+1}}\right)
\]
the regret we pay when using any $\gammab_r$ is at most $
    2\sqrt{(2^b\ln K) 2^r}
$. The largest $r$ we need is
$\bigl\lceil\log_2\overline{Q}^{(b)}\bigr\rceil$ and
\[
    \sum_{r=0}^{\lceil \log_2\overline{Q}^{(b)}\rceil} 2^{r/2} < 5\sqrt{\overline{Q}^{(b)}}~.
\]
Since we pay regret at most $1$ for each restart, we get
\[
    \E\bigl[L_{A,T} - L_{k,T}\bigr]
\le
    c\,\sum_{b=0}^{\lfloor\log_2 K\rfloor} \E\left[\sqrt{(\ln K)\left(2^b|\Tb| + \frac{1}{2}\sum_{t\in\Tb} Q^{(b)}_t\right)} + \bigl\lceil\log_2\overline{Q}^{(b)}\bigr\rceil\right]~.
\]
for some positive constant $c$.
Taking into account that
\begin{align*}
    \sum_{b=0}^{\lfloor\log_2 K\rfloor} 2^b|\Tb| &\le 2\sum_{t=1}^{T} |R_t|
\\
    \sum_{b=0}^{\lfloor\log_2 K\rfloor} \sum_{t \in \Tb} Q^{(b)}_t &= \sum_{t=1}^T Q_t^{(b_t)}
\\
    \sum_{b=0}^{\lfloor\log_2 K\rfloor} \bigl\lceil\log_2\overline{Q}^{(b)}\bigr\rceil
&= \mathcal{O}\bigl((\ln K)\ln(KT)\bigr)
\end{align*}
we obtain
\begin{align*}
    \E\bigl[L_{A,T} - L_{k,T}\bigr]
&\le
    c\,\sum_{b=0}^{\lfloor\log_2 K\rfloor}\E\left[\sqrt{(\ln K)\left(2^b|\Tb| + \frac{1}{2}\sum_{t\in\Tb} Q^{(b)}_t\right)}\right] + \mathcal{O}\bigl((\ln K)\ln(KT)\bigr)
\\ &\le
    c\,\lfloor\log_2 K\rfloor \E\left[\sqrt{\frac{\ln K}{\lfloor\log_2 K\rfloor}\sum_{t=1}^T\left(2|R_t| + \frac{1}{2} Q^{(b_t)}_t\right)}\right] + \mathcal{O}\bigl((\ln K)\ln(KT)\bigr)
\\
&=
{\mathcal O}\left((\ln K)\,\E\left[\sqrt{\sum_{t=1}^T \left(4|R_t|
+ Q_t^{(b_t)}\right)}\right] + (\ln K) \ln(KT)\right)
\end{align*}
%
as desired. 

\subsection{Proof of Theorem~\ref{c:final}}\label{s:appendix:c:final}


The following graph-theoretic lemma turns out to be fairly useful for
analyzing directed settings. It is a directed-graph counterpart to a
well-known result~\cite{c79,w81} holding for undirected graphs.
%
\begin{lemma}\label{l:amlemma}
Let $G = (V,D)$ be a directed graph,
with $V = \{1,\ldots,K\}$.
Let $d_i^-$ be the indegree of node $i$, and $\alpha = \alpha(G)$ be the
independence number of $G$. Then
\[
\sum_{i=1}^K \frac{1}{1+d_i^-} \leq 2\alpha\,\ln\left(1+\frac{K}{\alpha}\right)~.
\]
\end{lemma}
\begin{proof}
We proceed by induction, starting from the original $K$-node graph $G = G_K$
with indegrees $\{d_{i}^-\}_{i=1}^K = \{d_{i,K}^-\}_{i=1}^K$, and
independence number $\alpha = \alpha_K$, and then progressively reduce $G$ by
eliminating nodes and incident (both departing and incoming) arcs, thereby
obtaining a sequence of smaller and smaller graphs $G_K, G_{K-1}, G_{K-2},
\ldots $, associated indegrees $\{d_{i,K}^-\}_{i=1}^{K}$,
$\{d_{i,K-1}^-\}_{i=1}^{K-1}$, $\{d_{i,K-2}^-\}_{i=1}^{K-2}$, \ldots, and
independence numbers $\alpha_K, \alpha_{K-1}, \alpha_{K-2}, \ldots$.
Specifically, in step $s$ we sort nodes $i = 1, \ldots, s$ of $G_s$ in
nonincreasing value of $d_{i,s}^-$, and obtain $G_{s-1}$ from $G_s$ by
eliminating node $1$ (i.e., the one having the largest indegree among the
nodes of $G_s$), along with its incident arcs. On all such graphs, we use the
classical Turan's theorem~(e.g., \cite{as04}) stating that any {\em
undirected} graph with $n_s$ nodes and $m_s$ edges has an independent set of
size at least $\frac{n_s}{\frac{2m_s}{n_s}+1}$. This implies that if $G_s =
(V_s,D_s)$, then $\alpha_s$ satisfies\footnote { Note that $|D_s|$ is at
least as large as the number of edges of the undirected version of $G_s$
which the independence number $\alpha_s$ actually refers to. }
\begin{equation}\label{e:turan}
\frac{|D_s|}{|V_s|} \geq \frac{|V_s|}{2\alpha_s} - \frac{1}{2}~.
\end{equation}
We then start from $G_K$. We can write
\[
d_{1,K}^- = \max_{i=1\ldots K} d_{i,K}^- \geq \frac{1}{K}\,\sum_{i=1}^K d_{i,K}^- = \frac{|D_K|}{|V_K|}
\geq \frac{|V_K|}{2\alpha_K} - \frac{1}{2}~.
\]
%
Hence,
\begin{eqnarray*}
\sum_{i=1}^K \frac{1}{1+d_{i,K}^-}
&=& \frac{1}{1+d_{1,K}^-} + \sum_{i=2}^K \frac{1}{1+d_{i,K}^-} \\
&\leq&
\frac{2\alpha_K}{\alpha_K+K} + \sum_{i=2}^K \frac{1}{1+d_{i,K}^-}\\
&\leq&
\frac{2\alpha_K}{\alpha_K+K} + \sum_{i=1}^{K-1} \frac{1}{1+d_{i,K-1}^-}
\end{eqnarray*}
where the last inequality follows from $d_{i+1,K}^- \geq d_{i,K-1}^-$, $i =
1, \ldots K-1$, due to the arc elimination trasforming $G_K$ into $G_{K-1}$.
Recursively applying the same argument to $G_{K-1}$ (i.e., to the sum
$\sum_{i=1}^{K-1} \frac{1}{1+d_{i,K-1}^-}$), and then iterating all the way
to $G_1$ yields the upper bound
\[
\sum_{i=1}^K \frac{1}{1+d_{i,K}^-} \leq \sum_{i=1}^K \frac{2\alpha_i}{\alpha_i+i}~.
\]
Combining with $\alpha_i \leq \alpha_K = \alpha$, and $\sum_{i=1}^K
\frac{1}{\alpha+i} \leq \ln \left(1+\frac{K}{\alpha} \right)$ concludes the
proof.
\end{proof}

The next lemma
relates the size $|R_t|$ of the dominating set $R_t$ computed by the Greedy
Set Cover algorithm of~\cite{Chv79}, operating on the time-$t$ feedback
system $\{S_{i,t}\}_{i\in V}$, to the independence number $\alpha(G_t)$ and
the domination number $\gamma(G_t)$ of $G_t$.

\begin{lemma}\label{l:greedycover}
Let $\{S_i\}_{i \in V}$ be a feedback system, and $G = (V,D)$ be the induced
directed graph,
with vertex set $V = \{1,\ldots,K\}$, independence number $\alpha =
\alpha(G)$, and domination number $\gamma = \gamma(G)$. Then the dominating
set $R$ constructed by the Greedy Set Cover algorithm (see Section
\ref{s:prel}) satisfies
\[
|R| \le \min\bigl\{ \gamma(1+\ln K), \lceil 2\alpha\ln K \rceil + 1\bigr\}~.
\]
\end{lemma}
\begin{proof}
As recalled in Section~\ref{s:prel}, the Greedy Set Cover algorithm
of~\cite{Chv79} achieves $|R| \le \gamma(1+\ln K)$. In order to prove the
other bound, consider the sequence of graphs $G = G_1, G_2,\dots$, where each
$G_{s+1} = (V_{s+1},D_{s+1})$ is obtained by removing from $G_s$ the vertex
$i_s$ selected by the Greedy Set Cover algorithm, together with all the
vertices in $G_{s}$ that are dominated by $i_s$, and all arcs incident to
these vertices. By definition of the algorithm, the outdegree $d_s^+$ of
$i_s$ in $G_s$ is largest in $G_s$. Hence,
\[
    d_s^+ \ge \frac{|D_s|}{|V_s|} \ge \frac{|V_s|}{2\alpha_s} - \frac{1}{2} \ge \frac{|V_s|}{2\alpha} - \frac{1}{2}
\]
by Turan's theorem (e.g.,~\cite{as04}), where $\alpha_s$ is the independence
number of $G_s$ and $\alpha\ge \alpha_s$. This shows that
\[
   |V_{s+1}| = |V_s| - d_s^+ - 1 \le |V_s|\left(1 - \frac{1}{2\alpha}\right) \le |V_s|e^{-1/(2\alpha)}~.
\]
Iterating, we obtain $|V_s| \le K\,e^{-s/(2\alpha)}$. Choosing $s = \lceil
2\alpha\ln K \rceil+1$ gives $|V_s| < 1$, thereby covering all nodes. Hence
the dominating set $R = \{i_1, \ldots, i_s\}$ so constructed satisfies $|R|
\leq \lceil 2\alpha\ln K \rceil+1$.
\end{proof}

\begin{lemma}\label{l:ancillary}
If $a, b \geq 0$, and $a+b \geq B > A > 0$, then
\[
\frac{a}{a+b-A} \leq \frac{a}{a+b} + \frac{A}{B-A}~.
\]
\end{lemma}
\begin{proof}
\[
\frac{a}{a+b-A} - \frac{a}{a+b} = \frac{aA}{(a+b)(a+b-A)} \leq \frac{A}{a+b-A} \leq \frac{A}{B-A}~.
\]
\end{proof}

We now lift Lemma \ref{l:amlemma} to a more general statement.
\begin{lemma}\label{l:weightedamlemma}
Let $G = (V,D)$ be a directed graph,
with vertex set $V = \{1,\ldots,K\}$, and arc set $D$.
Let
$\alpha$ be the independence number of $G$, $R \subseteq V$ be a dominating
set for $G$ of size $r=|R|$,
and $p_1, \ldots, p_K$ be a probability distribution defined over $V$, such
that $p_i \geq \beta > 0$, for $i \in R$. Then
\[
\sum_{i=1}^K \frac{p_i}{p_i+ \sum_{j \,:\, j \reach{} i}\ p_j}
\leq
2\alpha\,\ln\left(1+\frac{\lceil\frac{K^2}{r\beta}\rceil+K}{\alpha}\right) + 2r~.
\]
\end{lemma}
\begin{proof}
The idea is to appropriately discretize the probability values $p_i$, and
then upper bound the discretized counterpart of $\sum_{i=1}^K \frac{p_i}{p_i+
\sum_{j \,:\, j \reach{} i}\ p_j}$ by reducing to an expression that can be
handled by Lemma \ref{l:amlemma}. In order to make this discretization
effective, we need to single out the terms $\frac{p_i}{p_i+ \sum_{j \,:\, j
\reach{} i}\ p_j}$ corresponding to nodes $i \in R$. We first write
\begin{eqnarray}
\sum_{i=1}^K \frac{p_i}{p_i+ \sum_{j \,:\, j \reach{} i}\ p_j}
&=&
\sum_{i\in R} \frac{p_i}{p_i+ \sum_{j \,:\, j \reach{} i}\ p_j} + \sum_{i\notin R} \frac{p_i}{p_i+ \sum_{j \,:\, j \reach{} i}\ p_j}\nonumber\\
&\leq&
r + \sum_{i\notin R} \frac{p_i}{p_i+ \sum_{j \,:\, j \reach{} i}\ p_j} \label{e:prelim}
\end{eqnarray}
and then focus on~(\ref{e:prelim}).

Let us discretize the unit interval\footnote { The zero value is not of our
concern here, because if $p_i = 0$, then the corresponding term
in~(\ref{e:prelim}) can be disregarded. } $(0,1]$ into subintervals
$\bigl(\tfrac{j-1}{M},\tfrac{j}{M}\bigr]$, $j = 1, \ldots, M$, where $M =
\lceil\frac{K^2}{r\beta}\rceil$. Let $\hp_i = j/M$ be the discretized version
of $p_i$, being $j$ the unique integer such that $ \hp_i - 1/M < p_i \leq
\hp_i $.
%
We focus on a single node $i\notin R$ with indegree $d_i^-$. Introduce the
shorthand notations $P_i = \sum_{j \,:\, j \reach{} i}\ p_j$ and $\hP_i =
\sum_{j \,:\, j \reach{} i}\ \hp_j$. We have that $\hP_i \geq P_i \geq
\beta$, since $i$ is dominated by some node $j \in R \cap N_i^-$ such that
$p_j \geq \beta$. Moreover, $P_i > \hP_i - \frac{d_i^-}{M} \geq \beta -
\frac{d_i^-}{M} >0$, and $\hp_i+\hP_i \geq \beta$. Hence, for any fixed node
$i\notin R$, we can write
\begin{eqnarray*}
\frac{p_i}{p_i+ P_i}
&\leq&
 \frac{\hp_i}{\hp_i+ P_i}\\
&<&
 \frac{\hp_i}{\hp_i+ \hP_i -  \frac{d_i^-}{M}}\\
&\leq&
 \frac{\hp_i}{\hp_i+ \hP_i} + \frac{d_i^-/M}{\beta - d_i^-/M}\\
&=&
 \frac{\hp_i}{\hp_i+ \hP_i} + \frac{d_i^-}{\beta M - d_i^-}\\
&< &
 \frac{\hp_i}{\hp_i+ \hP_i} + \frac{r}{K-r}
\end{eqnarray*}
where in the second-last inequality we used Lemma~\ref{l:ancillary} with $a =
\hp_i$, $b= \hP_i$, $A = d_i^-/M$, and $B = \beta > d_i^-/M$.
Recalling~(\ref{e:prelim}), and summing over $i$ then gives
\begin{equation}\label{e:discretization}
\sum_{i=1}^K \frac{p_i}{p_i+ P_i} \leq r + \sum_{i \notin R} \frac{\hp_i}{\hp_i + \hP_i} + r =
\sum_{i \notin R} \frac{\hp_i}{\hp_i + \hP_i} + 2r~.
\end{equation}
Therefore, we continue by bounding from above the right-hand side
of~(\ref{e:discretization}). We first observe that
\begin{equation}\label{e:discretization2}
\sum_{i\notin R} \frac{\hp_i}{\hp_i + \hP_i} = \sum_{i\notin R} \frac{\hs_i}{\hs_i + \hS_i}
\qquad\text{and}\qquad
\hS_i = \sum_{j \,:\, j \reach{} i} \hs_j
\end{equation}
where $\hs_i = M\hp_i$, $i = 1, \ldots, K$, are integers. Based on the
original graph $G$, we construct a new graph $\hG$ made up of connected
cliques. In  particular:
\begin{itemize}
\item Each node $i$ of $G$ is replaced in $\hG$ by a clique $C_i$ of size
    $\hs_i$; nodes within $C_i$ are connected by length-two cycles.
\item If arc $(i,j)$ is in $G$, then for each node of $C_i$ draw an arc
    towards each node of $C_j$.
\end{itemize}
We would like to apply Lemma~\ref{l:amlemma} to $\hG$. Note that, by the
above construction:
\begin{itemize}
\item The independence number of $\hG$ is the same as that of $G$;
\item The indegree $\hd_k^-$ of each node $k$ in clique $C_i$ satisfies
    $\hd_k^- = \hs_i-1 + \hS_i$.
\item The total number of nodes of $\hG$ is
\[
\sum_{i=1}^K \hs_i = M\sum_{i=1}^K \hp_i < M\sum_{i=1}^K \left(p_i + \frac{1}{M}\right) = M+K~.
\]
\end{itemize}
Hence, we can apply Lemma~\ref{l:amlemma} to $\hG$ with indegrees $\hd_k^-$,
and find that
\[
\sum_{i\notin R} \frac{\hs_i}{\hs_i + \hS_i} =
\sum_{i\notin R} \sum_{k \in C_i} \frac{1}{1+\hd_k^-}
\leq \sum_{i=1}^K \sum_{k \in C_i} \frac{1}{1+\hd_k^-}
\leq 2\alpha\ln\left(1+\frac{M+K}{\alpha}\right)~.
\]
Putting together~(\ref{e:discretization}) and~(\ref{e:discretization2}), and
recalling the value of $M$ gives the claimed result.
\end{proof}

\subsubsection*{Proof of Theorem~\ref{c:final}}
We are now ready to derive the proof of the theorem. We start from the upper
bound~(\ref{eq:doublingtrick}) in the statement of Lemma~\ref{thm:alg}. We
want to bound the quantities $|R_t|$ and $Q_t^{(b_t)}$ occurring therein at
any step $t$ in which a restart does not occur ---the regret for the time
steps when a restart occurs is already accounted for by the term
$\mathcal{O}\bigl((\ln K)\ln(KT)\bigr)$ in~(\ref{eq:doublingtrick}). Now,
Lemma~\ref{l:greedycover} gives
\[
|R_t| = \mathcal{O}\bigl(\alpha(G_t)\ln K\bigr)~.
\]
If $\gamma_t = \gamma^{(b_t)}_t$ for any time $t$ when a restart does not
occur, it is not hard to see that $\gamma_t = \Omega\bigl(\sqrt{(\ln
K)/(KT)}\bigr)$. Moreover, Lemma~\ref{l:weightedamlemma} states that
\[
Q_t = \mathcal{O}\bigl(\alpha(G_t)\ln(K^2/\gamma_t) + |R_t|\bigr) = \mathcal{O}\bigl(\alpha(G_t)\ln(K/\gamma_t)\bigr)~.
\]
Hence,
\[
Q_t = \mathcal{O}\bigl(\alpha(G_t)\ln(KT)\bigr).
\]
Putting together as in (\ref{eq:doublingtrick}) gives the desired result. 

\section{Technical lemmas and proofs from Section~\ref{sec:elpp}}\label{s:appendixelpp}
Once again, throughout this appendix $\E_t[\,\cdot\,]$ denotes the
conditional expectation $\E_t[\,\cdot\,|\, I_1, I_2, \ldots, I_{t-1}]$.
Moreover, as we did in previous appendices,
we first condition on the history $I_1, I_2, \ldots, I_{t-1}$, and then take
an expectation with respect to that history.


\subsection{Proof of Theorem \ref{thm:mainhighprob}}\label{subsec:thmmainhighprob}

The following lemmas are of preliminary importance in order to understand the
behavior of the ELP.P algorithm. Recall that for a directed graph $G =
(V,D)$, with vertex set $V = \{1,\ldots,K\}$, and arc set $D$, we write $\{
j\,:\, j \reach{}i \}$ to denote the set of nodes $j$ which are in-neighbors
of node $i$, where it is understood that node $i$ is an in-neighbor of
itself. Similarly,  $\{ j\,:\, i \reach{}j \}$ is the out-neighborhood of
node $i$ where, again, node $i$ is an out-neighbor of itself.
Let $\Delta_K$ be the $K$-dimensional probability simplex.

\begin{lemma}\label{lem:maxratio}
Consider a directed graph $G = (V,D)$, with vertex set $V = \{1,\ldots,K\}$,
and arc set $D$. Let $\mas(G)$ be the size of a largest acyclic subgraph of
$G$.
%
If $s_1,\ldots,s_K$ is a solution to the linear program
\begin{equation}
\label{eq:linprog}
\max_{(s_1,\ldots,s_K) \in \Delta_K} \min_{i \in V} \left(\sum_{j\,:\, j \reach{}i} s_j\right)
\end{equation}
then we have
%
\[
\max_{i \in V} \frac{1}{ \sum_{j\,:\, j \reach{} i } s_j} \leq \mas(G)~.
\]
%
\end{lemma}

\begin{proof}
We first show that the above inequality holds when the right-hand side is
replaced by $\gamma(G)$, the domination number of $G$. Let then $R$ be a
smallest (minimal) dominating set of $G$, so that $|R|=\gamma(G)$. Consider
the valid assignment $s_i = \Ind{i\in R}/\gamma(G)$ for all $i \in V$. This
implies that for all $i$, $\sum_{j\,:\, j \reach{}i} s_j \geq 1/\gamma(G)$,
because any $i \in V$ is either in $R$ or is dominated by a node in $R$.
Therefore, for this particular assignment, we have
\[
\max_{i\in V} \frac{1}{\sum_{j\,:\, j \reach{}i}s_j} \leq \gamma(G)~.
\]
The assignment returned by the linear program might be different, but it can
only make the left-hand side above smaller,\footnote{ This can be seen by
noting that~(\ref{eq:linprog}) is equivalent to
\[
\min_{(s_1,\ldots,s_K) \in \Delta_K} \max_{i\in V} \frac{1}{\sum_{j\,:\, j \reach{}i}s_j}
\]
} so the inequality still holds. Finally, $\gamma(G) \leq \mas(G)$ because
any set $M \subseteq V$ of nodes belonging to a maximal acyclic subgraph of
$G$ is itself a dominating set for $G$. In fact, assuming the contrary, let
$j$ be any node such that $j \notin M$. Then, including $j$ in $M$ would
create a cycle (because of the maximality of $M$), implying that $j$ is
already dominated by some other node in $M$.
\end{proof}

%

\begin{lemma}\label{lem:inequalities}
Consider a directed graph $G = (V,D)$, with vertex set $V = \{1,\ldots,K\}$,
and arc set $D$. Let $\mas(G)$ be the size of a largest acyclic subgraph of
$G$. Let $(p_1, \ldots, p_K) \in \Delta_K$ and $(s_1, \ldots, s_K) \in
\Delta_K$ satisfy
\[
\sum_{i=1}^{K}\frac{p_{i}}{\sum_{j\,:\, j \reach{}i} p_{j}} \leq \mas(G) \qquad\text{and}\qquad
\max_{i \in V} \frac{1}{\sum_{j\,:\, j \reach{}i} s_{j}} \leq \mas(G)
\]
with $p_i \geq \gamma s_i$, $i \in V$, for some $\gamma > 0$.
Then the following relations hold:
\begin{enumerate}
\item 
\[
\sum_{i=1}^{K}\frac{p_{i}}{\left(\sum_{j\,:\, j \reach{}i}p_{j}\right)^2} \leq \frac{\mas^2(G)}{\gamma}~;
\]
\item 
\[
\sum_{i=1}^K p_i\,\sum_{j\,:\,i \reach{}j} \frac{p_j}{\sum_{\ell\,:\, \ell \reach{}j} p_{\ell}}= 1~;
\]
\item 
\[
\sum_{i=1}^K p_i\,\sum_{j\,:\,i \reach{}j} \frac{p_j}{\left(\sum_{\ell\,:\, \ell \reach{}j} p_{\ell}\right)^2}
\leq \mas(G)~;
\]
\item 
\[
\sum_{i=1}^K p_i\,\left(\sum_{j\,:\,i \reach{}j}
                                                \frac{p_j}{\sum_{\ell\,:\, \ell \reach{}j} p_{\ell}}\right)^2
\leq \mas(G)~;
\]
\item 
\[
\sum_{i=1}^K p_i\,\left(\sum_{j\,:\,i \reach{}j}
                                                \frac{p_j}{\left(\sum_{\ell\,:\, \ell \reach{}j} p_{\ell}\right)^2}\right)^2
\leq \frac{\mas^3(G)}{\gamma}~.
\]
\end{enumerate}
\end{lemma}

\begin{proof}
Let us introduce the shorthand $q_{i} = \sum_{j\,:\, j \reach{}i} p_j$, for
$i \in V$.
\begin{enumerate}
\item 
We apply H\"{o}lder's inequality, and the assumptions of this lemma to
obtain
\begin{align*}
\sum_{i=1}^K \frac{p_i}{q_i^2}
&~=~ \sum_{i=1}^K \left(\frac{p_i}{q_i}\right)\,
        \left(\frac{1}{q_i}\right)\\
&~\leq~ \left(\sum_{i=1}^K \frac{p_i}{q_i}\right)\,
        \left(\max_{i\in V} \frac{1}{q_i}\right)\\
 &~=~ \left(\sum_{i=1}^K \frac{p_i}{q_i}\right)\,
        \left(\max_{i\in V} \frac{1}{\sum_{j\,:\, j \reach{}i}p_j}\right)\\
&~\leq~ \mas(G)\,\max_{i\in V} \frac{1}{\gamma\left(\sum_{j\,:\, j \reach{}i}s_j\right)}\\
&~\leq~ \frac{\mas^2(G)}{\gamma}~.
\end{align*}
\item 
We have
\[
\sum_{i=1}^K \sum_{j\,:\,i \reach{}j} \frac{p_i\,p_j}{q_j}
~=~ \sum_{j=1}^{K}\frac{p_j\,q_j}{q_j} ~=~ \sum_{j=1}^{K} p_j ~=~ 1~.
\]
\item 
Similar to the previous item, we can write
\[
\sum_{i=1}^K \sum_{j\,:\,i \reach{}j} \frac{p_i\,p_j}{q_j^2}
~=~ \sum_{j=1}^{K}\frac{p_j\,q_j}{q_j^2}
~=~ \sum_{j=1}^{K} \frac{p_j}{q_j}
~\leq~ \mas(G)~.
\]
%
%
\item 
From Item 2, and the assumptions of this lemma, we can write
\begin{align*}
\sum_{i=1}^K p_i\,\left(\sum_{j\,:\,i \reach{}j} \frac{p_j}{q_j}\right)^2
&~=~
\sum_{i=1}^K \left(p_i\,\sum_{j\,:\,i \reach{}j} \frac{p_j}{q_j}\right)\,
\left( \sum_{j\,:\,i \reach{}j} \frac{p_j}{q_j} \right)\\
&~\leq~\left(\sum_{i=1}^K p_i\,\sum_{j\,:\,i \reach{}j} \frac{p_j}{q_j}\right)\,
\left(\max_{i\in V} \sum_{j\,:\,i \reach{}j} \frac{p_j}{q_j} \right)\\
&~\leq~
\left(\sum_{i=1}^K p_i\,\sum_{j\,:\,i \reach{}j} \frac{p_j}{q_j}\right)\,
\left( \sum_{j=1}^K \frac{p_j}{q_j} \right)\\
&~=~
\sum_{j=1}^K \frac{p_j}{q_j}
~\leq~
\mas(G)~.
\end{align*}
\item 
From Item 1 and Item 3 above, we can write
\begin{align*}
\sum_{i=1}^K p_i\,\left(\sum_{j\,:\,i \reach{}j}
                                                \frac{p_j}{q_j^2}\right)^2
&~=~
\sum_{i=1}^K \left(p_i\,\sum_{j\,:\,i \reach{}j}
                                                \frac{p_j}{q_j^2}\right)\,
\left(\sum_{j\,:\,i \reach{}j} \frac{p_j}{q_j^2}\right)\\
&~\leq~
\left(\sum_{i=1}^K p_i\,\sum_{j\,:\,i \reach{}j}
                                                \frac{p_j}{q_j^2}\right)\,
\left(\max_{i\in V}\sum_{j\,:\,i \reach{}j} \frac{p_j}{q_j^2}\right)\\
&~\leq~
\left(\sum_{i=1}^K p_i\,\sum_{j\,:\,i \reach{}j} \frac{p_j}{q_j^2}\right)\,
\left(\sum_{i=1}^{K}\frac{p_{i}}{q_i^2} \right)\\
&~\leq~
\mas(G)\,\,\frac{\mas^2(G)}{\gamma}\\
&~=~  \frac{\mas^3(G)}{\gamma}
\end{align*}
\end{enumerate}
concluding the proof.
\end{proof}

Lemma~\ref{lem:inequalities} applies, in particular, to the distributions
$s_t = (s_{1,t}, \ldots, s_{K,t})$ and $p_t = (p_{1,t}, \ldots, p_{K,t})$
computed by  ELP.P at round $t$. The condition for $p_t$ follows from
Lemma~\ref{lemma:nDGA}, while the condition for $s_t$ follows from
Lemma~\ref{lem:maxratio}. In other words, putting together
Lemma~\ref{lemma:nDGA} and Lemma~\ref{lem:maxratio} establishes the following
corollary.

\begin{corollary}
\label{cor:elp:sp} Let $p_t=(p_{1,t}, \ldots, p_{K,t}) \in \Delta_K$ and
$s_t=(s_{1,t}, \ldots, s_{K,t}) \in \Delta_K$ be the distributions generated
by ELP.P at round $t$. Then,
\[
\sum_{i=1}^{K}\frac{p_{i,t}}{\sum_{j\,:\, j \reach{t}i} p_{j,t}} \leq \mas(G) \qquad\text{and}\qquad
\max_{i \in V} \frac{1}{\sum_{j\,:\, j \reach{t}i} s_{j,t}} \leq \mas(G)~,
\]
with $p_{i,t} \geq \gamma_t\,s_{i,t}$, for all $i = 1, \ldots, K$.
\end{corollary}

For the next result, we need the following version of Freedman's inequality
\cite{Freedman75} (see also \cite[Lemma A.8]{cbl06}).


\begin{lemma}\label{lem:freedman}
Let $X_1,\ldots,X_T$ be a martingale difference sequence with respect to the
filtration $\{\mathcal{F}_t\}_{t=1,\ldots,T}$, and with $|X_i|\leq B$ almost
surely for all $i$. Also, let $V>0$ be a fixed upper bound on
$\sum_{t=1}^{T}\E\bigl[X_t^2\mid\mathcal{F}_{t-1}\bigr]$. Then for any
$\delta\in (0,1)$, it holds with probability at least $1-\delta$
\[
\sum_{t=1}^{T}X_t \leq \sqrt{2\ln\left(\frac{1}{\delta}\right)V}+\frac{B}{2}\ln\left(\frac{1}{\delta}\right).
\]
\end{lemma}

\begin{lemma}\label{lem:pg}
Let $\{a_t\}_{t=1}^{T}$ be an arbitrary sequence of positive numbers, and let
$s_t = (s_{1,t}, \ldots, s_{K,t})$ and $p_{t} = (p_{1,t}, \ldots, p_{K,t})$
be the probability distributions computed by ELP.P at the $t$-th round. Then,
with probability at least $1-\delta$,
\begin{align}
\sum_{t=1}^{T}\sum_{i=1}^{K}a_t& p_{i,t}(\hgain_{i,t}-\gain_{i,t})\notag\\
&\leq~
\sqrt{2\ln\left(\frac{1}{\delta}\right)\sum_{t=1}^{T}a_t^2\,\mas(G_t)}+\frac{1}{2}\ln\left(\frac{1}{\delta}\right)
\max_{t = 1...T} \bigl(a_t\,\mas(G_t)\bigr)+\beta\sum_{t=1}^{T}a_t\,\mas(G_t)~.\label{e:crossprod}
\end{align}
\end{lemma}
\begin{proof}
Recall that $q_{i,t} = \sum_{j\,:\,j\reach{t} i} p_{j,t}$, for $i \in V$, and
let
\[
\tilde{\gain}_{i,t} = \frac{\gain_{i,t}\,\Ind{i \in S_{I_t,t}}}{q_{i,t}}
\]
with $g_{i,t} = 1-\ell_{i,t}.$ Note that $\hgain_{i,t}$ in
Figure~\ref{alg:bandits} satisfies $ \hgain_{i,t} = \tilde{\gain}_{i,t} +
\tfrac{\beta}{q_{i,t}} $, so that we can upper bound the left-hand side
of~(\ref{e:crossprod}) by
\[
\sum_{t=1}^{T}\sum_{i=1}^{K}a_tp_{i,t}(\tilde{\gain}_{i,t}-\gain_{i,t})+\beta\sum_{t=1}^{T}
\sum_{i=1}^{K}\frac{a_tp_{i,t}}{q_{i,t}}
\]
which by Corollary~\ref{cor:elp:sp} is at most
\begin{equation}\label{eq:pg1}
\sum_{t=1}^{T}\sum_{i=1}^{K}a_tp_{j,t}(\tilde{\gain}_{i,t}-\gain_{i,t})+\beta\sum_{t=1}^{T}a_t\,\mas(G_t)~.
\end{equation}
It is easy to verify that
$\sum_{i=1}^{K}a_tp_{i,t}(\tilde{\gain}_{i,t}-\gain_{i,t})$ is a martingale
difference sequence (indexed by $t$), since $\tilde{\gain}_{i,t}$ is an
unbiased estimate of $\gain_{i,t}$ conditioned on the previous rounds.
Moreover,
\[
\sum_{i=1}^{K}a_tp_{i,t}\left(\tilde{\gain}_{i,t}-\gain_{i,t}\right)
~=~
\sum_{i=1}^{K}a_tp_{i,t}(\frac{\Ind{i \in S_{I_t,t}}}{q_{i,t}}-1) \gain_{i,t}
~\leq~
 a_t\sum_{i=1}^{K}\frac{p_{i,t}}{q_{i,t}}
~\leq~  \max_{t = 1,\dots,T} a_t\,\mas(G_t)
\]
and
%
\begin{align*}
\E_t \left[\left(\sum_{i=1}^{K}a_t p_{i,t}(\tilde{\gain}_{i,t}-\gain_{i,t})\right)^2\right]
&~\leq~ a_t^2\,\E_t \left[\left(\sum_{i=1}^{K}p_{i,t}\tilde{\gain}_{i,t}\right)^2\right]\\
&\leq a_t^2\sum_{i=1}^{K} p_{i,t} \left(\sum_{j\,:\,i\reach{t}j}p_{j,t}\frac{1}{q_{j,t}}\right)^2\\
&~\leq~  a_t^2\,\mas(G_t)
\end{align*}
by \lemref{lem:inequalities}, item 4. Therefore, by invoking
\lemref{lem:freedman}, we get that with probability at least $1-\delta$,
\[
\sum_{t=1}^{T}\sum_{j=1}^{K}a_tp_{j,t}(\tilde{\gain}_{j,t}-\gain_{j,t})
\leq
\sqrt{2\ln\left(\frac{1}{\delta}\right)\sum_{t=1}^{T}a_t^2\,\mas(G_t)}
+\frac{1}{2}\ln\left(\frac{1}{\delta}\right)
\max_{t = 1,\dots,T} a_t\,\mas(G_t)~.
\]
Substituting into \eqref{eq:pg1}, the lemma follows.
\end{proof}

\begin{lemma}\label{lem:pg2}
Let $s_t = (s_{1,t}, \ldots, s_{K,t})$ and $p_{t} = (p_{1,t}, \ldots,
p_{K,t})$ be the probability distributions computed by  ELP.P, run with
$\beta \leq 1/4$, at round $t$. Then, with probability at least $1-\delta$,
\begin{align*}
\sum_{t=1}^{T}\sum_{i=1}^{K}p_{i,t}\hgain_{i,t}^2
\leq
\sum_{t=1}^{T}\left(\frac{\beta^2\,\mas^2(G_t)}{\gamma_t}+2\,\mas(G_t)\right)
&+\sqrt{2\ln\left(\frac{1}{\delta}\right)\sum_{t=1}^{T}
\left(\frac{4\beta^2\,\mas^4(G_t)}{\gamma_t^2}
+\frac{3\,\mas^3(G_t)}{\gamma_t}\right)}\\
&+\ln\left(\frac{1}{\delta}\right)\max_{t = 1,\dots,T}\, \frac{\mas^2(G_t)}{\gamma_t}~.
\end{align*}
\end{lemma}
\begin{proof}
Recall that $q_{i,t} = \sum_{j\,:\,j\reach{t} i} p_{j,t}$, for $i \in V$. By
the way we defined $\hgain_{i,t}$ and \lemref{lem:inequalities}, item 1, we
have that
\[
\sum_{i=1}^{K}p_{i,t}\hgain_{i,t}^2
~\leq~
\sum_{i=1}^{K}p_{i,t}\left(\frac{1+\beta}{q_{i,t}}\right)^2
~\leq~ \frac{(1+\beta)^2 \mas^2(G_t)}{\gamma_t}~.
\]
Moreover, from $\gain_{i,t} \leq 1$, and again using
\lemref{lem:inequalities}, item 1, we can write
\begin{align*}
\E_t\left[\left(\sum_{j=1}^{K} p_{j,t}\hgain_{j,t}^2\right)^2\right]
&~\leq~
\sum_{i=1}^{K} p_{i,t}\left(\sum_{j=1}^{K}\frac{p_{j,t}}{\left(q_{j,t}\right)^2}
\left(\Ind{i \reach{t} j}+\beta\right)^2\right)^2\\
&=
\sum_{i=1}^{K} p_{i,t}\left(\beta^2\sum_{j=1}^{K}\frac{p_{j,t}}{\left(q_{j,t}\right)^2}
+(2 \beta + 1)\sum_{j\,:\,i \reach{t} j}\frac{p_{j,t}}{\left(q_{j,t}\right)^2}\right)^2\\
&\leq
\sum_{i=1}^{K} p_{i,t}\left(\frac{\beta^2\,\mas^2(G_t)}{\gamma_t}
+(2 \beta+1)\sum_{j\,:\,i \reach{t} j}\frac{p_{j,t}}{\left(q_{j,t}\right)^2}\right)^2
\end{align*}
which by expanding, using \lemref{lem:inequalities}, items 3 and 5, and
slightly simplifying, is at most
\[
%
\frac{(\beta^4+2\beta^2(2\beta+1))\,\mas^4(G_t)}{\gamma_t^2}+\frac{(2\beta+1)^2\,\mas^3(G_t)}{\gamma_t}
\leq
\frac{4\beta^2\,\mas^4(G_t)}{\gamma_t^2}+\frac{3\,\mas^3(G_t)}{\gamma_t}
\]
the last inequality exploiting the assumption $\beta \leq 1/4$. Invoking
\lemref{lem:freedman} we get that with probability at least $1-\delta$
\begin{align}
\sum_{t=1}^T\sum_{i=1}^{K}p_{i,t}\hgain_{i,t}^2
&\leq
\sum_{t=1}^{T}\sum_{i=1}^{K}p_{i,t}\E_t[\hgain_{i,t}^2]
+\sqrt{2\ln\left(\frac{1}{\delta}\right)
\sum_{t=1}^{T}
\left(
\frac{4\beta^2\,\mas^4(G_t)}{\gamma_t^2}+\frac{3\,\mas^3(G_t)}{\gamma_t}
\right)}\notag\\
&~~~~+\frac{(1+\beta)^2}{2}\ln\left(\frac{1}{\delta}\right)\max_{t=1\ldots T} \frac{\mas^2(G_t)}{\gamma_t}~.\label{eq:pg2}
\end{align}
Finally, from $g_{i,t} \leq 1$, \lemref{lem:inequalities}, item 1, and the
assumptions of this lemma, we have
\begin{align*}
\sum_{i=1}^{K}p_{i,t}\E_t[\hgain_{i,t}^2]
&\leq
\sum_{i=1}^{K}p_{i,t}\sum_{j=1}^{K}p_{j,t}\left(\frac{\Ind{j \reach{t} i}+\beta}{q_{i,t}}\right)^2\\
&= \beta^2\sum_{i=1}^{K}\frac{p_{i,t}}{\left(q_{i,t}\right)^2}
   +(2 \beta +1)\sum_{i=1}^{K} p_{i,t}\sum_{j\,:\,j\reach{t}i}\frac{p_{j,t}}{\left(q_{i,t}\right)^2}\\
&= \beta^2\sum_{i=1}^{K}\frac{p_{i,t}}{\left(q_{i,t}\right)^2}
   +(2 \beta +1)\sum_{i=1}^{K} \frac{p_{i,t}}{q_{i,t}}\\
&\leq \frac{\beta^2\,\mas^2(G_t)}{\gamma_t}+(2 \beta+1)\,\mas(G_t)\\
&\leq \frac{\beta^2\,\mas^2(G_t)}{\gamma_t}+2\,\mas(G_t)
\end{align*}
where we used again $\beta\leq 1/4$.
Plugging this back into \eqref{eq:pg2} the result follows.
\end{proof}

\begin{lemma}\label{lem:g}
Suppose that the ELP.P algorithm is run with $\beta \leq 1/4$. Then it holds
with probability at least that $1-\delta$ that for any $i=1,\ldots,K$,
\[
\sum_{t=1}^{T}\hgain_{i,t}\geq \sum_{t=1}^{T}\gain_{i,t}- \frac{\ln(K/\delta)}{\beta}~.
\]
\end{lemma}
\begin{proof}
Let $\lambda>0$ be a parameter to be specified later, and recall that $\E_t$
denotes the expectation at round $t$, conditioned on all previous rounds.
Since $\exp(x)\leq 1+x+x^2$ for $x\leq 1$, we have by definition of
$\hgain_{i,t}$ that
\begin{align*}
\E_t&\Bigl[\exp\left(\lambda(g_{i,t}-\hgain_{i,t})\right)\Bigr]
= \E_t\left[\exp\left(\lambda\left( g_{i,t}-\frac{\gain_{i,t} \Ind{I_t\reach{t}i}}{q_{i,t}}\right)
                                                              -\frac{\beta \lambda}{q_{i,t}}\right)\right]\\
&\leq
\left(1+\E_t\left[\lambda\left( g_{i,t}-\frac{\gain_{i,t}\Ind{I_t\reach{t}i}}{q_{i,t}}\right)\right]
+
\E_t\left[\left(\lambda\left( g_{i,t}-\frac{\gain_{i,t} \Ind{I_t\reach{t}i}}{q_{i,t}}\right)\right)^2\right]\right)\exp\left(-\frac{\beta \lambda}{q_{i,t}}\right)\\
&\leq
\left(1+0+\lambda^2\E_t\left[\left(\frac{\gain_{i,t}\Ind{I_t\reach{t}i}}{q_{i,t}}\right)^2\right]\right)
\exp\left(-\frac{\beta \lambda}{q_{i,t}}\right)\\
&\leq
\left(1+\lambda^2 \sum_{j\,:\,j \reach{t} i}\frac{p_{j,t}}{\left(q_{i,t}\right)^2}\right)
\exp\left(-\frac{\beta \lambda}{q_{i,t}}\right)\\
&=
\left(1+\frac{\lambda^2 }{q_{i,t}}\right)
\exp\left(-\frac{\beta \lambda}{q_{i,t}}\right)~.
\end{align*}
Picking $\lambda = \beta$, and using the fact that $(1+x)\exp(-x)\leq 1$, we
get that this expression is at most $1$. As a result, we have
\[
\E\left[\exp\left(\lambda\sum_{t=1}^{T}\left(\gain_{i,t}-\hgain_{i,t}\right)\right)\right] \leq 1~.
\]
Now, by a standard Chernoff technique, we know that for any $\lambda>0$,
\[
\Pr\left(\sum_{t=1}^{T}\left(\gain_{i,t}-\hgain_{i,t}\right)>\epsilon\right)
\leq \exp(-\lambda \epsilon)\E\left[\exp\left(\lambda\sum_{t=1}^{T}\left(\gain_{i,t}-\hgain_{i,t}\right)\right)\right]~.
\]
In particular, for our choice of $\lambda$, we get the bound
\[
\Pr\left(\sum_{t=1}^{T}\left(\gain_{i,t}-\hgain_{i,t}\right)>\epsilon\right) \leq \exp\left(-\beta\epsilon\right)~.
\]
Substituting $\delta=\exp(-\beta \epsilon)$, solving for $\epsilon$, and
using a union bound to make the result hold simultaneously for all $i$, the
result follows.
\end{proof}

\subsubsection*{Proof of Theorem~\ref{thm:mainhighprob}}

With these key lemmas at hand, we can now turn to prove
Theorem~\ref{thm:mainhighprob}.
%
%
We have
\begin{equation}\label{eq:pbegin}
\frac{W_{t+1}}{W_t}
~=~
\sum_{i\in V}\frac{w_{i,t+1}}{W_t}
 ~=~
\sum_{i\in V}\frac{w_{i,t}}{W_t}\,\exp(\eta \hgain_{i,t})~.
\end{equation}
%

Now, by definition of $q_{i,t}$ and $\gamma_t$ in Algorithm~\ref{alg:bandits}
we have
\[
q_{i,t} \geq \gamma_t\,\sum_{j\,:\,j\reach{t}i} s_{j,t} \geq (1+\beta)\,\eta
\]
for all $i \in V$, so that
\[
\eta \hgain_{j,t} ~\leq~  \eta \max_{i \in V} \left(\frac{1+\beta}{q_{i,t}}\right)
~\leq~ 1~.
\]
%
Using the definition of $p_{i,t}$ and the inequality $\exp(x)\leq 1+x+x^2$
for any $x\leq 1$, we can then upper bound the right-hand side
of~(\ref{eq:pbegin}) by
\begin{align*}
\sum_{i\in V}\frac{p_{i,t}-\gamma_t s_{i,t}}{1-\gamma_t}\left(1+\eta\hgain_{i,t}+\eta^2\hgain_{i,t}^2\right)
~\leq~
1+\frac{\eta}{1-\gamma_t}\sum_{i\in V}p_{i,t}\hgain_{i,t}+\frac{\eta^2}{1-\gamma_t}\sum_{i=1}^{K}p_{i,t}\hgain_{i,t}^2~.
\end{align*}
Taking logarithms and using the fact that $\ln(1+x)\leq x$, we get
\[
\ln\left(\frac{W_{t+1}}{W_t}\right)
~\leq~
\frac{\eta}{1-\gamma_t}\sum_{i\in V}p_{i,t}\hgain_{i,t}+\frac{\eta^2}{1-\gamma_t}\sum_{i\in V}p_{i,t}\hgain_{i,t}^2~.
\]
Summing over all $t$, and canceling the resulting telescopic series, we get
\begin{equation}\label{eq:pupbound}
\ln\left(\frac{W_{T+1}}{W_1}\right)
~\leq~ \sum_{t=1}^{T}\sum_{i\in V}\frac{\eta}{1-\gamma_t}p_{i,t}\hgain_{i,t}
+\sum_{t=1}^{T}\sum_{i\in V}\frac{\eta^2}{1-\gamma_t}p_{i,t}\hgain_{i,t}^2.
\end{equation}
Also, for any fixed action $k$, we have
\begin{equation}\label{eq:plowbound}
\ln\left(\frac{W_{T+1}}{W_1}\right) \geq \ln\left(\frac{w_{j,T+1}}{W_1}\right)
= \eta\sum_{t=1}^{T}\hgain_{k,t}-\ln K~.
\end{equation}
Combining \eqref{eq:pupbound} with \eqref{eq:plowbound}, and slightly
rearranging and simplifying, we get
\begin{align}
\sum_{t=1}^{T}&\hgain_{k,t} - \sum_{t=1}^{T}\sum_{i\in V}p_{i,t}\hgain_{i,t} \notag\\
&\leq
\frac{\ln K}{\eta}+\frac{\eta}{\displaystyle 1-\max_{t = 1,\dots,T} \gamma_t}\sum_{t=1}^{T}\sum_{i\in V}p_{i,t}\hgain_{i,t}^2
+\frac{1}{\displaystyle 1-\max_{t = 1,\dots,T} \gamma_t}\sum_{t=1}^{T}\sum_{i\in V}\gamma_t p_{i,t}\hgain_{i,t}~.\label{eq:porbound}
\end{align}
We now start to apply the various lemmas, using a union bound. To keep things
manageable, we will use asymptotic notation to deal with second-order terms.
In particular, we will use $\widetilde{\Ocal}$ notation to hide numerical
constants and logarithmic factors\footnote { Technically,
$\widetilde{\Ocal}(f)=O(f\log^{O(1)} f)$. In our $\widetilde{\Ocal}$ we also
ignore factors that depend logarithmically on $K$ and $1/\delta$.
}. By definition of $\beta$ and $\gamma_t$, as well as Corollary
\ref{cor:elp:sp}, it is easy to verify\footnote { The bound for $\beta$ is by
definition. The bound for $\gamma_t$ holds since by \lemref{lem:maxratio} and
the assumptions that $\eta\leq 1/(3K)$ and $\beta\leq 1/4$ we have
\[
\gamma_t=\frac{(1+\beta)\eta}{\min_{i\in V}\sum_{j\,:\, j \reach{t} i } s_ {j,t}}
\leq (1+\beta)\eta\,\mas(G_t)\leq \frac{(1+\beta)\max(G_t)}{3K}\leq \frac{1+1/4}{3}< 1/2~.
\]
} that
\begin{equation}\label{e:conditions}
\beta = \widetilde{\Ocal}(\eta) \qquad \gamma_t = \widetilde{\Ocal}(\eta\,\mas(G_t)) \qquad \gamma_t \in \left[\eta,\frac{1}{2}\right]~.
\end{equation}
First, by \lemref{lem:pg}, we have with probability at least $1-\delta$ that
\begin{equation}\label{eq:final1}
\sum_{t=1}^{T}\sum_{i=1}^{K} p_{i,t}\hgain_{i,t} \leq \sum_{t=1}^{T}\sum_{i=1}^{K}p_{i,t}\gain_{i,t}
+
\sqrt{2\ln\left(\frac{1}{\delta}\right)\sum_{t=1}^{T}\mas(G_t)}+\beta\sum_{t=1}^{T}\mas(G_t)
+\widetilde{\Ocal}\left(\max_{t = 1...T}~\mas(G_t)\right)~.
\end{equation}
Moreover, by Azuma's inequality, we have with probability at least $1-\delta$
that
\begin{equation}\label{eq:final2}
\sum_{t=1}^{T}\sum_{i=1}^{K}p_{i,t}\gain_{i,t}\leq \sum_{t=1}^{T}\gain_{I_t,t}+\sqrt{\frac{\ln(1/\delta)}{2}\,T}~.
\end{equation}
Second, again by \lemref{lem:pg} and the conditions~(\ref{e:conditions}), we
have with probability at least $1-\delta$ that
\begin{align}
\sum_{t=1}^{T}\sum_{i=1}^{K}\gamma_t p_{i,t}\hgain_{i,t}
~&\leq~
\sum_{t=1}^{T}\sum_{i=1}^{K}\gamma_t p_{i,t}\gain_{i,t}
     +\widetilde{\Ocal}\left(\max_{t = 1,\dots,T} \mas^2(G_t)(1+\sqrt{T}\eta+T\eta^2)\right)\notag\\
~&\leq~
\sum_{t=1}^{T}\gamma_t+\widetilde{\Ocal}\left(\max_{t = 1,\dots,T} \mas^2(G_t)(1+\sqrt{T}\eta+T\eta^2)\right)~.\label{eq:final3}
\end{align}
Third, by \lemref{lem:pg2}, and conditions~(\ref{e:conditions}), we have with
probability at least $1-\delta$ that for all $i$,
\begin{equation}\label{eq:final4}
\sum_{t=1}^{T}p_{i,t}\hgain_{i,t}^2 \leq 2\sum_{t=1}^{T}\mas(G_t)
+\left(\max_{t = 1,\dots,T}(\mas^2(G_t))\right)\widetilde{\Ocal}\left(T\eta+\frac{1}{\eta}
+\sqrt{T\left(1+\frac{1}{\eta}\right)}\right)~.
\end{equation}
Fourth, by \lemref{lem:g}, we have with probability at least $1-\delta$ that
\begin{equation}\label{eq:final5}
\sum_{t=1}^{T}\hgain_{k,t}\geq \sum_{t=1}^{T}\gain_{k,t}- \frac{\ln(K/\delta)}{\beta}~.
\end{equation}
%


Combining
\eqref{eq:final1},\eqref{eq:final2},\eqref{eq:final3},\eqref{eq:final4} and
\eqref{eq:final5} with a union bound (i.e., replacing $\delta$ by
$\delta/5$), substituting back into \eqref{eq:porbound}, and slightly
simplifying, we get that with probability at least $1-\delta$,
$\sum_{t=1}^{T}(\gain_{k,t}-\gain_{I_t,t})$ is at most
\begin{align*}
&\sqrt{2\ln\left(\frac{5}{\delta}\right)\sum_{t=1}^{T}\mas(G_t)}
+\beta\sum_{t=1}^{T}\mas(G_t)+\sqrt{\frac{\ln(5/\delta)}{2}T}
+\frac{\ln(5K/\delta)}{\beta}+\frac{\ln K}{\eta}\\
&+4\eta\sum_{t=1}^{T}\mas(G_t)+2\sum_{t=1}^{T}\gamma_t
+(1+\sqrt{T\eta}+T\eta^2)\,\widetilde{\Ocal}\left(\max_{t = 1,\dots,T}(\mas^2(G_t))\right)~.
\end{align*}
Substituting in the values of $\beta$ and $\gamma_t$, overapproximating, and
simplifying (in particular, using Corollary \ref{cor:elp:sp} to upper bound
$\gamma_t$ by $(1+\beta)\,\eta\,\mas(G_t)$), we get the upper bound
\begin{align*}
\sqrt{5\ln\left(\frac{5}{\delta}\right)\sum_{t=1}^{T}\mas(G_t)}
&+\frac{2\ln (5K/\delta)}{\eta} + 12\eta\,\sqrt{\frac{\ln(5K/\delta)}{\ln K}}\,\sum_{t=1}^{T}\mas(G_t)
\\ &+
\widetilde{\Ocal}(1+\sqrt{T\eta}+T\eta^2)\left(\max_{t = 1,\dots,T} (\mas^2(G_t))\right)~.
\end{align*}
%
%
In particular, by picking $\eta$ such that
\[
\eta^2 = \frac{1}{6}\,\frac{\sqrt{\ln(5K/\delta)\,(\ln K)}}{\sum_{t=1}^{T} m_t}
\]
noting that this implies $\eta = \widetilde{\Ocal}(1/\sqrt{T})$, and
overapproximating once more, we get the second bound
\[
\sum_{t=1}^{T}(\gain_{k,t}-\gain_{I_t,t})
\leq
10\,\frac{\ln^{1/4}(5K/\delta)}{\ln^{1/4}K}\,\sqrt{\ln \left(\frac{5K}{\delta}\right)\sum_{t=1}^{T} m_t}
+\widetilde{\Ocal}(T^{1/4})\left(\max_{t = 1,\dots,T} \mas^2(G_t) \right)~.
\]
To conclude, we simply plug in $\ell_{i,t}=1-\gain_{i,t}$ for all $i$ and
$t$, thereby obtaining the claimed results.

\end{document}